\documentclass[letterpaper]{article} % DO NOT CHANGE THIS
\usepackage{aaai24}  % DO NOT CHANGE THIS
\usepackage{times}  % DO NOT CHANGE THIS
\usepackage{helvet}  % DO NOT CHANGE THIS
\usepackage{courier}  % DO NOT CHANGE THIS
\usepackage[hyphens]{url}  % DO NOT CHANGE THIS
\usepackage{graphicx} % DO NOT CHANGE THIS
\usepackage{natbib}  % DO NOT CHANGE THIS AND DO NOT ADD ANY OPTIONS TO IT
\usepackage{caption} % DO NOT CHANGE THIS AND DO NOT ADD ANY OPTIONS TO IT
\usepackage{algorithm}
\usepackage{multirow}
\usepackage{graphicx}
\usepackage[table,xcdraw]{xcolor}
\usepackage{algorithmic}
\usepackage{color,xcolor}
\usepackage{amssymb}
\usepackage{bbding}
\usepackage{threeparttable}
\usepackage{booktabs}
\usepackage{multirow}
\usepackage{graphicx}
\usepackage{subfigure}
\usepackage{multicol}
\usepackage{lipsum}
\usepackage{multirow}
\usepackage{graphicx}
\usepackage[table,xcdraw]{xcolor}

\usepackage{amsmath}
\usepackage{amssymb}
\usepackage{mathtools}
\usepackage{amsthm}
\theoremstyle{plain}
\newtheorem{theorem}{Theorem}
\newtheorem{observe}{Observation}
\usepackage{xargs}                      % Use more than one optional parameter in a new commands
\usepackage{newfloat}
\usepackage{listings}
\DeclareCaptionStyle{ruled}{labelfont=normalfont,labelsep=colon,strut=off} % DO NOT CHANGE THIS
\floatstyle{ruled}
\newfloat{listing}{tb}{lst}{}
\floatname{listing}{Listing}
\newcommand{\RomanNumeralCaps}[1]{\MakeUppercase{\romannumeral #1}}

\title{A New Baseline Assumption of Integated
Gradients Based on Shaply value}

\author{Shuyang Liu\textsuperscript{\rm 1*}, 
    Zixuan Chen\textsuperscript{\rm 1*},
    Ge Shi\textsuperscript{\rm 2}, 
    Ji Wang\textsuperscript{\rm 2}, 
    Changjie Fan\textsuperscript{\rm 3},
    Yu Xiong\textsuperscript{\rm 3},
    Runze Wu\textsuperscript{\rm 3}\\
    Yujing Hu\textsuperscript{\rm 3},
    Ze Ji\textsuperscript{\rm 4},
    Yang Gao\textsuperscript{\rm 1}\\
}
\affiliations{{\normalsize{$^1$ Nanjing University, Nanjing, China}}\\
{\normalsize{$^2$ University of California, Davis, California, USA}}\\
{\normalsize{$^3$ NetEase Fuxi AI Lab, Hangzhou, China}}\\
{\normalsize{$^4$ Cardiff University, Cardiff, U.K.}}}

\begin{document}

\maketitle

\footnote{*These authors contributed to the work equally and should be regarded as co-first authors.}

% \begin{abstract}
% Numerous approaches have attempted to interpret deep neural networks (DNNs) by attributing the prediction of DNN to its input features. 
% One of the well-studied attribution methods is Integrated Gradients (IG). Specifically, the choice of baselines for IG is a critical consideration for generating meaningful and unbiased explanations for model predictions in different scenarios. However, current practice of exploiting a single baseline fails to fulfill this ambition, thus demanding multiple baselines.
% Fortunately, the inherent connection between IG and Aumann-Shapley Value forms a unique perspective to rethink the design of baselines. Under certain hypothesis, we theoretically analyse that a set of baseline aligns with the coalitions in Shapley Value. 
% Thus, we propose a novel baseline construction method called Shapley Integrated Gradients (SIG) that searches for a set of baselines by proportional sampling to partly simulate the computation path of Shapley Value.
% Simulations on GridWorld show that SIG approximates the proportion of Shapley Values.
% Furthermore, experiments conducted on various image tasks demonstrate that compared to IG using other baseline methods, SIG exhibits an {\it improved} estimation of feature's contribution, offers more {\it consistent} explanations across diverse applications, and is {\it generic} to distinct data types or instances with insignificant computational overhead.
% \end{abstract}

\begin{abstract}
    Numerous attempts have been made to understand deep neural networks (DNNs) by linking their predictions to input features. Integrated Gradients (IG) is a prominent method in this realm. The selection of baselines in IG plays a crucial role in providing meaningful and unbiased explanations of model predictions across various contexts. However, the conventional practice of using a single baseline often falls short, necessitating the use of multiple baselines. The intrinsic relationship between IG and the Aumann-Shapley Value offers a fresh perspective on designing baselines. We theoretically establish that, under specific hypotheses, a set of baselines corresponds to coalitions in the Shapley Value framework. Consequently, we introduce a new method for constructing baselines, termed Shapley Integrated Gradients (SIG), which employs proportional sampling to approximate the Shapley Value computation pathway. Simulations in GridWorld confirm that SIG closely mirrors the distribution of Shapley Values. Additionally, experiments on various image processing tasks reveal that SIG, compared to other baseline methods in IG, provides a more accurate assessment of feature contributions, yields more consistent explanations across different applications, and adapts seamlessly to various data types with minimal computational increase.
\end{abstract}

\section{\RomanNumeralCaps{1}. Introduction}
% As a prominent branch of artificial intelligence, deep learning (DL) has demonstrated substantial achievements in a variety of fields, including computer vision, \textcolor{red}{reinforcement learning} \jw{(I understand why we want to have RL here, but these three are not in the same comparable level, RL is method, the other two are applications)}, and large language models. [cite] 
As artificial intelligence (AI) becomes more sophisticated and indispensable in our daily lives, the necessity of transparency and accountability in AI decision-making has grown significantly.~\cite{adadi2018peeking, gunning2019darpa} 
To this end, Explainable Artificial Intelligence (XAI) was proposed and has been increasingly studied within the community that focuses on developing machine learning models and systems that can offer understandable explanations for their decisions and predictions~\cite{arrieta2020explainable, das2020opportunities}. Research on XAI boomed exceptionally after the paramount success of deep neural networks (DNNs), as they are considered by many as complex yet intriguing ``black-box'' models.~\cite{angelov2020towards, ras2022explainable}
% However, the inherent trade-off between transparency and performance complicates understanding the fundamental mechanisms of DL~\cite{puiutta2020explainable}. 
% As a result, Explainable Artificial Intelligence (XAI) has increasingly garnered attention. 

Among numerous XAI methods, 
attributing the prediction of a deep network to its input features is particularly popular due to its intuitive explanations and broad applicability~\cite{simonyan2013deep, ancona2017towards, heskes2020causal}. 
% For instance, in an image classification task, an attribution method can pinpoint which pixels the deep model used for categorization. 
One of these attribution methods, Integrated Gradients (IG)~\cite{sundararajan2017axiomatic} stands out as a leading algorithm by integrating gradients of the model's output with respect to input along a straight path from baseline to input. The mechanism and easy-to-implement property of IG make it particularly useful and applicable for understanding how input features impact the model's output.
Since its introduction, IG has continuously evolved \cite{ghorbani2019interpretation,jha2020enhanced,yang2023idgi,pourdarbani2023interpretation, yang2023idgi}, language models' explanation \cite{enguehard2023sequential}, and electronic health \cite{duell2023batch}.

For IG, or path attribution methods in general, it is usually required to choose a hyperparameter known as the baseline input. 
Crucially, the choice of an appropriate baseline for IG plays a pivotal role in providing meaningful feature attributions and explanations~\cite{frye2020shapley}. Take image classification task as an example: the gradient at a pixel represents the direction, in which the function increases most rapidly, indicating the pixel's influence on the model's prediction. Since the contributions of players are measured by accumulating gradients on images interpolated between the baseline value and the current input, the choice of baseline has an outstanding influence on these gradients. 

There are three baseline inputs that are commonly adopted in XAI: random \cite{jha2020enhanced}, zero~\cite{ancona2019explaining, aas2021explaining}, and mean \cite{dabkowski2017real}, all of which are single baseline values.
% Unfortunately, none of these baselines 
We briefly give justifications of the pitfalls of choosing a single baseline value.
(\romannumeral1) A single baseline that is naturally considered as neutral and task unrelated by human might not have a ``near-zero'' score as recommended by~\cite{sundararajan2017axiomatic} for many models. For instance, for the function of interest $F_1(x) = x+1$, if a baseline of zero is used, namely $x_0 = 0$, then $F_1(x_0) = 1$ is not ``near zero'', thus making interpretability implausible.
(\romannumeral2) A single baseline that's meticulously crafted for a specific model is not generic to all instances. A straightforward example is $F_2(x) = x^3-6x^2+11x-6$ where $x_1 = 1, x_2 = 2, x_3 = 3$ are the solutions to $F_2(x) = 0$. Choosing any one of them as the baseline might introduce incompleteness and unintended bias into the attributions. The caveats of utilising single baselines are also discussed and validated by~\citet{chen2022explaining}.

Intuitively, the challenges associated with a single baseline motivate us to exploit a set of baselines for IG. We therefore delve deep into IG method, rethinking how to construct a set of well-interpretable and consistent baselines. We traced back to the underlying scheme of IG. 
Although applied in distinct fields, as an attribution method, IG corresponds to a ``shortcut'' approximation of a cost-sharing method in economics called Aumann-Shapley Value~\cite{aumann1974values} where
the model, base features, and attributions of IG are analogous to the cost function, players, and cost-shares of Aumann-Shapley Value respectively.~\cite{sundararajan2020many}
% Shapley Value is a frequently used method in attribute-based approaches because it provides an accurate measure of each player's contribution to the game by aggregating the marginal contributions of all coalitions. 
Furthermore, we show that under certain assumptions, coalitions in Shapley Values involving marginal contributions can be treated as baselines and explained samples. Specifically, we assume that the baselines correspond to the coalitions in Shapley Values. This means that we are able to select a set of baselines to model the computation of the Shapley Values in a concise and efficient manner, echoing our objective of utilizing a set of baselines.
Therefore, based on these assumptions we propose a novel baseline construction method called {\it \textbf{S}hapley \textbf{I}ntegrated \textbf{G}radients} ({\bf SIG}). SIG assumes baseline corresponds to coalitions in Shapley Value and search for a set of baseline by proportional sampling to partly simulate the computation path of Shapley Value.
To validate the effectiveness, genericnesss and consistency of SIG, we conduct experiments in two main types of environments according to the different nature of task inputs: (1) simulation of Shapley Values in GridWorld maze environment \cite{sutton2018reinforcement} to demonstrate that SIG approximates the proportion of Shapley Values; and (2) a variety of image input tasks, including image classification \cite{he2016deep}, and expression codes \cite{zhang2021learning}, to reveal that SIG produces enhanced explanation and more consistent interpretation results across a variety of tasks.

% We conduct experiments within two part: 1) simulation of Shapley Value in a maze environment to demonstrate that our method can partly modify calculation of Integrated Gradients; 2) versatility on image input tasks, including Atari, image classification and expression code, to reveal our method yields superior, more consistent interpretative results on various tasks.

% We conduct experiments within a reinforcement learning maze environment to demonstrate that Shapley Integrated Gradients can effectively simulate Shapley Value. Additionally, we execute experiments on reinforcement learning Atari environments and computer vision tasks such as image classification and facial expression coding. The experimental results reveal that compared to other conventional baselines, our method yields superior, more consistent interpretative results, and demonstrates impressive versatility.

To summarize, our contributions are three-folds: (\romannumeral1) we analyse the drawbacks of using a single baseline for IG and hence a set of baselines is desirable; we then rethink baseline design for IG from the perspective of Shapley Value, and theoretically show that under certain hypothesis, the coalitions in Shapley Values can be regarded as a set of baselines for IG; (\romannumeral2) we propose a novel baseline construction method named SIG that partially models the computational path of the Shapley Value by finding the set of baselines; (\romannumeral3) The experimental results demonstrate that the proposed SIG provides a new perspective on IG, showing improved and more consistent interpretation than existing methods in a wide array of applications, and is generic to various data types and instances, thus making it of considerable value for practical applications.

% The rest of this paper is organized as follows: 1) we introduce relevant research concerning Integrated Gradients and the Shapley Value; 2) we then provide a unique perspective about shortcut of Integrated Gradients through Shapley Value, and based on this understanding, we present our proposed method; 3) finally, we conduct various experiments to demonstrate the advantages of our approach.

% Organization of the rest of the paper is as follows: 1) we will introduce some related work about Integrated Gradients and Shapley Value in Related ; 2) we will explain shortcut of Integrated Gradients from perpecitive of Shapley Value and based on it, provide our method; 3) we conduct experiments on various task to prove strength of our work.

\section{\RomanNumeralCaps{2}. Related Work}
% While most prior studies centered on identifying an ideal baseline, our proposed method seeks a set of baselines corresponding to the coalitions in Shapley Value. These baselines are not only informative but also simpler to construct.

\paragraph{Integrated Gradients and the Choice of Its Baselines} Integrated Gradients (IG)~\cite{sundararajan2017axiomatic} fuses  
the implementation invariance of gradients with sensitivity-based techniques, necessitating a crucial baseline value, as noted by the authors. \citet{mudrakarta2018did} illuminated computation path of IG is the straight-line between baseline and input sample.

% \jw{\citet{mudrakarta2018did} illuminated the intuition behind Integrated Gradients: the gradient with respect to the input features at each point along a straight-line path from the baseline to the input being explained, is used to attribute the alteration in prediction probability back to the input features. The method then accumulates these gradients across the trajectory through a path integral. (is it needed?)}
Recent studies \cite{dabkowski2017real, merrick2020explanation,kumar2020problems,binder2016layer,shrikumar2017learning, frye2020shapley, tan2023maximum} have provided experiential advice on 
choosing baseline values, while falling short of theoretical illumination. 
For example, \citet{dabkowski2017real} opted to use the average of randomly selected samples from the dataset as the baseline; \citet{frye2020shapley} established baseline value of a pixel in relation to surrounding pixels; 
\citet{chen2021explaining} perceived the baseline as a representation of the background distribution, a viewpoint that harmonizes with our perspective; \citet{feng2022comparing} also highlights that the computational path of IG is not equivalent to that of the Shapley Value.

% Their proposed approach suggests that the rewards players receive should consist of both the rewards from the background of the game and those accruing from player participation in the game, a viewpoint that harmonizes with our perspective.

\paragraph{Shapley Value} \citeauthor {shapley1951notes}~\shortcite{shapley1951notes} proposed Shapley Value to allocate contributions to players in a cooperative setting. It is the only distribution with linearity, nullity, symmetry, and efficiency axioms. \citeauthor{aumann2015values}~\shortcite{aumann2015values} extended the concept of Shapley Value to infinite game. Prior attempts have been made to incorporate Shapley Values to attribution explanations: \citet{lundberg2017unified} proposed Shapley Additive exPlanations (SHAP), in which each feature of the model is viewed as a player in the game, the model itself is seen as a utility function, and the chain rule is employed to diminish computational complexity;~\citet{ghorbani2019data} enhanced the efficiency of Shapley Value estimation by using Monte Carlo Sampling and gradient-based methods; \citet{mitchell2022sampling}
proposed a new approach to mitigating the issue of slow convergence of standard Monte Carlo sampling sampling;
\citet{chen2022explaining} pointed out that utilizing multiple baselines mitigates the potential bias that can arise from attributions based on a single baseline. 

\section{\RomanNumeralCaps{3}. Preliminaries}
\paragraph{Notation} A table of notations used in this work is presented in the Appendix.

\paragraph{Shapley Value}
Consider a game with $n$ players and a utility function $v$, where $v : 2^N \rightarrow \mathbb{R}$ maps subsets of players to real numbers. Any set of players in $N$ is called a coalition $S$. For a given coalition $S$ and a player $x_i$ such that $x_i \notin S$, the marginal contribution of player $x_i$ to coalition $S$ is defined as $v(S\cup x_i) - v(S)$. The Shapley Value of player $x_i$, denoted by $f_{SV}(x_i)$, is then computed as the sum over all coalitions $S \in N/\{i\}$, weighted by the probability of selecting each coalition. Specifically,
\begin{equation}\label{Eq_1}
    f_{SV}(x_i)=\sum_{S\in N/\{i\}}\frac{|S|!(|N|-|S|-1)!}{|N|!}
    (v(S\cup x_i)-v(S))
\end{equation}

\paragraph{Aumann-Shapley Value}
With the extension to infinite game, Aumann-Shapley Value was proposed with the definition that $ds$ represents infinitely small player in game, $I$ represents the complete set of players and $tI$ is a perfect sampling, representing a proportion $t$ of all the players. Aumann-Shapley Value can be written as follows:
\begin{equation}\label{Eq_2}
    f_{SV}(ds)=\int^1_0v(tI+ds)-v(tI)dt
\end{equation}

\paragraph{Integrated Gradients}
Integrated Gradients (IG) is a method that was developed to attribute the prediction of a DNN to its input features. It integrates the gradient of the prediction concerning the input features over a straight-line path between the input $x$ and a baseline $x'$. The IG for a model $F$ is expressed as follows:
\begin{equation}\label{Eq_3}
    f_{IG}(x_i)=(x_i-x'_i)\times\int_{0}^1\frac{\partial F(x'+\alpha(x-x'))}{\partial x_i}d\alpha
\end{equation}

\section{\RomanNumeralCaps{4}. Method}
In this section, we will first present Intergrated Gradients (IG) from the perspective of Shapley Value. Through this view, we discover and analyze the shortcomings of IG: calculation path of IG takes a straight line shortcut compared to that of Shapley Value. Building upon this observation, we propose our baseline construction method, Shapley Integrated Gradients (SIG).

\subsection{Integrated Gradients From the Perspective of Shapley Value}
Given an explained sample $x = [x_1, \ldots, x_n]$, a baseline example $x' = [x'_1, \ldots, x'_n]$, and a function $F:R^n \rightarrow [0, 1]$ that represents a deep network, we prove that IG is associated with Aumann-Shapley Value as follows:

% TODO 修正假设前提
\begin{theorem} \label{Theorem}
    Suppose $x'$ represents the empty set $\emptyset$ with the absence of all features, $x$ represents the complete set $I$ with the presence of all features, and utility function of any sample $x$ is evaluated as $v(x) = F(x) - F(x')$, the integral of Integrated Graidents of all features is a simulation of the integral of Aumann-Shapley Values for all players when the model is evaluated along the linearly interpolated path between baseline and explained sample.
\end{theorem}  
\begin{proof}
We treat each feature $x_i$ of sample $x$ as a player in game theory and the mix up of $x$ and $x'$ in a feature-wise binary on-off manner as a coalition. Thus, the worth of perfect sample $tI$ including coalitions with $t$ proportion of $x$ and $1-t$ proportion of $x'$ is represented as $v(tI)$. The contribution of $ds$ to the coalition is $v(tI+ds)-v(tI)$. Aumann-Shapley Value computes the contribution of an infinitesimal player $ds$ by integrating the functional gain of adding the player to the perfect sample $tI$ of the all-player $I$ at all proportions $t\in[0, 1]$. When $t=0$ and $t=1$, we get an $\emptyset$ and a complete set $I$ respectively. Therefore, the worth of the complete set $I$ that represents the integral contribution of all players can be written as follows:

\begin{align*}
v(I)&=\int_{\emptyset}^{I}f_{SV}(ds)\\
&=\int_{\emptyset}^{I}\int_0^1 v(tI+ds)-v(tI)dt\\
&=\int_{\emptyset}^{I}\int_0^1\frac{v(tI+ds)-v(tI)}{ds}dt ds\\
&=\int_{\emptyset}^{I}\int_0^1\frac{d v(tI)}{d s}dtds
\end{align*}

Meanwhile, IG accumulates the contribution of each feature by integrating the partial derivatives of model $F$ with respect to the feature at points along a straight-line path $[x', \ldots, x'+\alpha(x-x'), \ldots x]$ from the baseline to the explained sample. As such, the integral of IG over all features from $\emptyset$ to $I$ can be written as follows:
\begin{align*}
\int_{\emptyset}^{I}f_{IG}(dx) &=\int_{\emptyset}^{I}(x_i-x_i') \int_{0}^1\frac{\partial F(x'+\alpha(x-x'))}{\partial x_i}d\alpha dx\\
&=\int_{0}^1\int_{\emptyset}^{I}(x_i-x_i')\frac{\partial F(x'+\alpha(x-x'))}{\partial x_i}dxd\alpha\\
&=\int_{0}^1 \nabla F(x'+\alpha(x-x')) d\alpha\\
&=F(x)-F(x') = v(I)
\end{align*} 
\end{proof}

Interestingly, indicated by the proof above, the integral of IG approaches the integral of Shapley Value but though a different computation path. Each feature in IG plays essentially the same role as a player in Shapley Value. Uppon this proof, we use player and feature interchangeably in the following paragraphs.

\subsection{Limitations of Integrated Gradients}
\begin{figure}[htbp]
   \centering
	\includegraphics[width=\columnwidth]{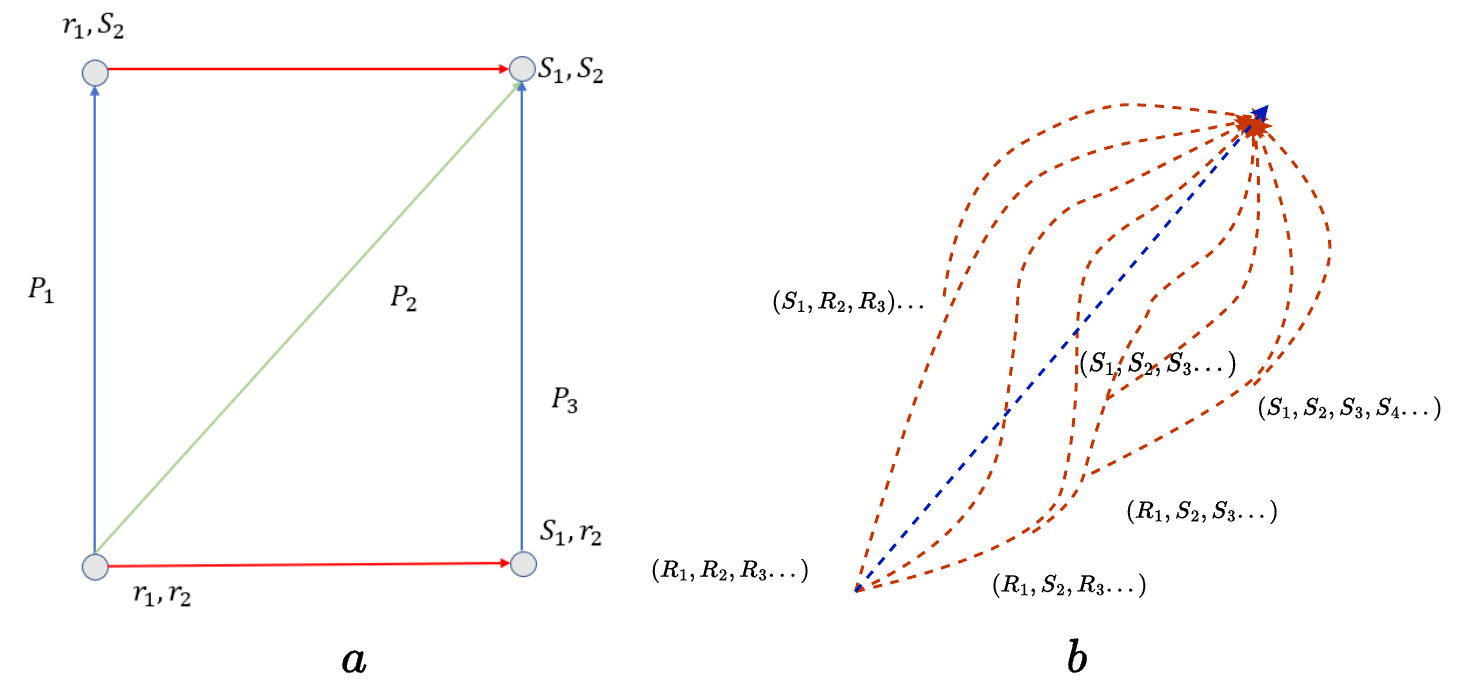}
	\caption{($a$) For a two-features input, red lines represent calculation of Shapley Value for feature $S_1$ and blue lines represent that for feature $S_2$. While Path $P_2$ is calculation path of IG. ($b$) Red paths are the calculation path of Shapley Value while blue path is the calculation path of IG for a $n$-features input.}
    \label{path}
% \vspace{-0.3cm}
\end{figure}

From the aforementioned Theorem{~\ref{Theorem}}, we derive that IG approaches the Aumann-Shapley Value in computing the contribution of the entire feature set. However, the goal of attribution methods is to compute the contribution of each individual feature. We illustrate their difference with a simple example.

As is shown in Fig.{~\ref{path}}($a$), we assume that there is a two-features input of the deep model, denoted as $x = (S_1, S_2)$ with its corresponding baseline as $x'=(r_1,r_2)$. Then the set of coalitions encompasses $\left(r_1,r_2\right)$, $\left(S_1,r_2\right)$, $\left(r_1,S_2\right)$, and $\left(S_1,S_2\right)$. The Shapley Value of features $S_1$ and $S_2$ are:

% As shown in Fig.{~\ref{path}}($a$), let's assume that there are two-dimensional features, where the baseline is denoted as $x'=(r_1,r_2)$ and input sample is represented as $x=(s_1,s_2)$. We consider feature $s_1$ and $s_2$ as two players and take the baseline $x'$ as the background, so the set of coalitions includes $\left(s_1,r_2\right)$, $\left(r_1,s_2\right)$, and $\left(s_1,s_2\right)$. The Shapley Value of player $s_1$ is noted as:
\begin{align*}
    f\left(S_1\right)&=\frac{\left(v\left(S_1,r_2\right)-v\left(r_1,r_2\right)\right)}{2}+\frac{\left(v\left(S_1,S_2\right)-v\left(r_1,S_2\right)\right)}{2} \\
    f\left(S_2\right)&=\frac{\left(v\left(r_1,S_2\right)-v\left(r_1,r_2\right)\right)}{2}+\frac{\left(v\left(S_1,S_2\right)-v\left(S_1,r_2\right)\right)}{2}
\end{align*}

% 需要再被修饰一下
From the computation of Shapley Value for $S_1$ and $S_2$, we observe that calculation paths traverse through all four points of coalitions, represented by the blue and red lines in Fig.{~\ref{path}}. The calculation path of Shapley Value can be denoted as $P_1$ and $P_3$. On the other hand, as the proof shows, we ascertain that IG actually computes contribution along the straight line $P_2$. It becomes apparent that calculation path of IG, expressed as $P_2$, takes a shortcut compared to path of Shapley Value, shown as $P_1$ and $P_3$.

% From calculation path of Shapley Value of player $s_2$, it's clear that the path passes through the points $\left(r_1,r_2\right)$, $\left(r_1,S_2\right)$, and $\left(S_1,S_2\right)$, depicted as path $P_1$.  However, according to Sundararajan et\cite{sundararajan2017axiomatic}, Integrated Gradients computes the contribution of a feature along path $P_2$.  The baseline $x'$ is considered devoid of any information and all points in path $P_2$ can be viewed as the complete set $I$ in the Aumann Shapley Value. In other words, Integrated Gradients computes the contribution of path $P_2$. We observe that calculation path of Integrated Gradients, expressed as path $P_2$, takes a shortcut compared to path of the Shapley Value, shown as path $P_1$.

% From the calculation path of Shapley Value of player $s_2$, it's clear that path pass through the points $\left(r_1,r_2\right),\left(r_1,S_2\right),\left(S_1,S_2\right)$, which is shown as path $P_1$. While Integrated Gradients computes feature's contribution alongside path $P_2$ according to \cite{sundararajan2017axiomatic}. Baseline $x'$ is considered as not containing any information and all points in path $P_2$ can be regarded as complete set $I$ in Aumann Shapley Value. In other words, Integrated Gradients computes the contribution of path $P_2$. We can find out that calculation path of Integrated Gradients, which is expressed as path $P_2$, takes a shortcut compared to that of Shapley Value, which is shown as path $P_1$. 

Extending to n-features setting, as shown in Fig.{~\ref{path}}($b$), the computation of Shapley Value takes many polyline paths while the computation of IG approximates it by taking a straight line short cut between baseline and the explained sample. Consequently, this discrepancy may result in inaccurate attributions of individual features. 

\subsection{Shapley Integrated Gradients} 
In above subsection, we observe that IG may lead to inaccuracy in estimating the contributions for features caused by the shortcut of calculation path compared to Shapley Value. However, because of the single path computation, it's efficiency is obvious.  Contrastively, Shapley Value is appreciated for its accuracy in determining individual contributions, its computational burden of leveraging all coalitions, subjected to $\mathcal{O}(2^n)$ where $n$ is the number of features, restricts its application. To take some benefits of both, we combine Shapley Value with IG and propose \textbf{\textit{Shapley Intergratd Gradients}} (SIG), a novel baseline construction method for improving IG from the perspective of Shapley Value. 

The intuition of the algorithm comes from Equation{~\ref{Eq_1}}. As shown, the Shapley Value of feature $i$ is the weighted average of marginal contributions of a player across all possible coalitions $\{S; S\in N/\{i\}\}$, where the weight $w_i(S)$ and marginal contribution $V_i(S)$ are defined as:

\begin{equation} \label{Eq_4}
w_i(S) = \frac{|S|!(|N|-|S|-1)!}{|N|!}
\end{equation}

\begin{equation} \label{Eq_5}
V_i(S) = v(S\cup x_i)-v(S)
\end{equation}

By inspecting the above two factors, here are some interesting observations:

\begin{observe} \label{Ob1} 
    Given a game $N$, $w_i(S)$ is only dependent on the number of players $k=|S|$ in the coalition $S$. Therefore, we can merge $\{w_i(S); \forall S \text{ s.t. } |S|=k \wedge \forall i, i \in N\}$ as a universal weight value $w(k)$. This value can be precomputed as long as we decide $k$.
\end{observe} 

\begin{observe} \label{Ob2} 
The sum of $\{w(k); \forall S \text{ s.t. } |S|=k \wedge \forall k, k \in [0, |N|-1]\}$ is $1$ since $w(k)$ represents the inverse value of the number of combinations $\binom{|N|-1}{|S|-1}$ times the number of different values of $k$ we can choose from the game. $k$ acts like the proportion $t$ in Equation{~\ref{Eq_2}} to measure the magnitude of coalitions. Thus, all coalitions are naturally categorized into $|N|$ groups subjected by the size $k$ and $\sum_{S} w(k) = 1/|N|$. 
\begin{align*}
\frac{1}{w(k)} &= \frac{|N|!}{|S|!(|N|-|S|-1)!} \\
&= |N| \times \frac{|N-1|!}{k!(|N|-1-k)!} \\
&= |N| \times \binom{|N|-1}{k}
\end{align*}
\end{observe} 

\begin{observe} \label{Ob3} 
As a direct conclusion of Theorem{~\ref{Theorem}}, the marginal contribution of adding $i;\forall i \in N$ to $S$ can be approximated by using IG from baseline $S$ to complete set $I$. 
\end{observe}

In the definition of Shapley Value in Equation{~\ref{Eq_1}}, the marginal contributions of of all the coalitions $S$ are all computed and then weighted by their corresponding $w_i(S)$ whose quantity is exponentially large. However, thanks to the observations above, we leverage a ``trick'' called \textbf{proportional sampling}, as is shown in Fig.{~\ref{shapley}}($a$), to circumvent it.

\begin{figure}[htbp!]
   \centering
	\includegraphics[width=\columnwidth]{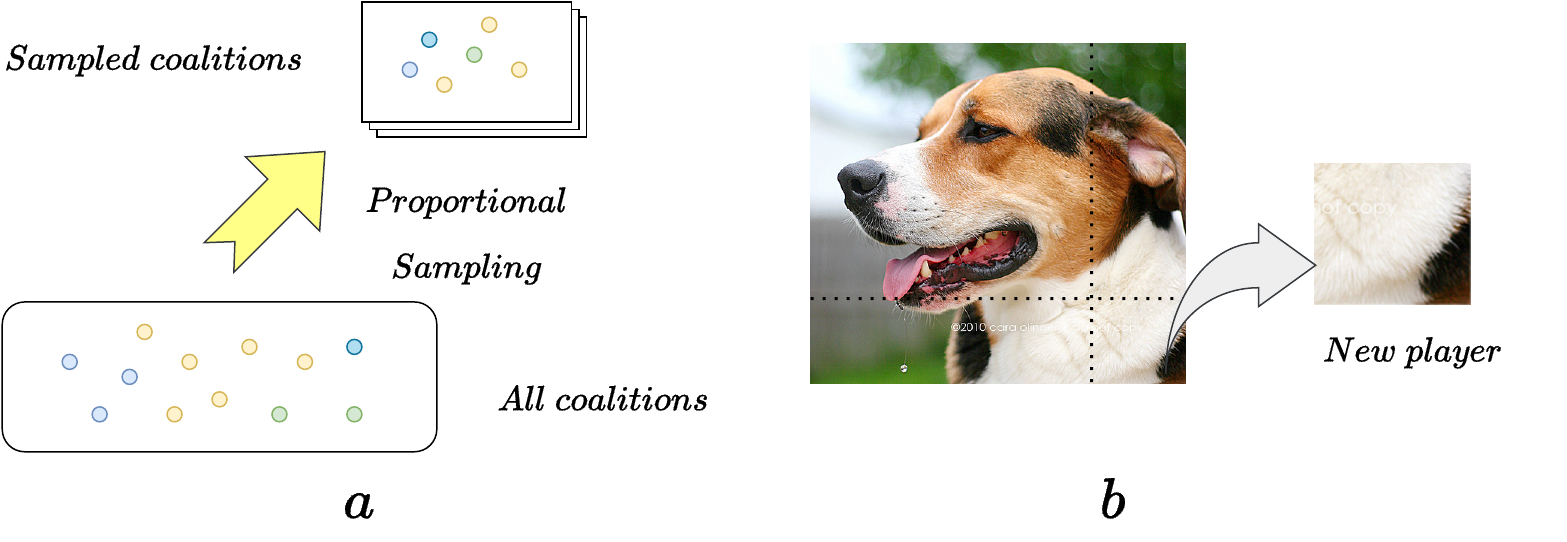}
    \caption{($a$) Proportional sampling in our SIG. Different colored nodes represent different weights defined by Equation{~\ref{Eq_4}}. ($b$) Construction of new players. A patch of pixels is considered as a new player/feature on which we search the baseline set.}
    \label{shapley}
% \vspace{-0.2cm}
\end{figure}

\paragraph{Proportional Sampling.} Instead of sampling coalitions from a binomial distribution and weight them by their corresponding $w_i(S)$, an alternative approach is to sample coalitions proportional to the precomputed $w_i(k)$ and weight them uniformly by 1.

Here, we justify the ``trick'' creates an unbiased estimator of Shapley Value briefly. Due to Observation{~\ref{Ob1}}, the weight of coalition $S$ can be precomputed once we select a $k$, and it's a constant no matter which specific coalition is drawn as long as it's size is $k$. Due to Observation{~\ref{Ob2}}, the expected marginal contribution of different sizes are equally treated for different $k$ values. Therefore, the prior importance of each coalition are universally 1. Suppose we draw an infinite large number of samples, the expectation of $V_i(S)$ drawn proportional to weight $w(k)$ is an infinite approach to that of $V_i(S)$ weighted by $w(k)$. The detailed proof is in the Appendix. 

In addition, Observation{~\ref{Ob3}} enables reducing the computation burden of computing the marginal contributions of all players $i$ one by one but complete them at once. Since IG is a gradient based approach which is supported by GPUs, leveraging it gives SIG more potential to speed up.  

To sum up, the three observations inspire our proposed SIG algorithm that constructs a set of sampled coalitions with different magnitude $k$ as baselines of IG to approximate Shapley Value in an efficient way. 

\subsection{SIG Algorithm} 
We posit that a coalition in Shapley Value corresponds to a baseline in IG. We aim to obtain relative Shapley Value since people typically care more about relative contributions of features rather than the absolute contributions. The overall algorithm, termed as \textit{Shapley Integrated Gradients} (SIG), is presented in Algorithm{~\ref{algo1}}.

\begin{algorithm}
\caption{Shapley Intergratd Gradients (SIG)}
\begin{algorithmic}
\STATE \textbf{INPUT}: explained sample $x$, default sample $x'$, model $F$
\STATE \textbf{PARAMETER}: player set $N$, sample size $B$
\STATE \textbf{OUTPUT}: approximated Shapley Values $\mathcal{V}$
\STATE Initialize baseline set $D = \emptyset$
\FOR{$k \gets 0$ to $N-1$}
    \STATE Compute weight $w(k)$ as Equation{~\ref{Eq_4}}
\ENDFOR
\STATE Normalize all weights $\hat{w}(k) = w(k) /  \sum_{k=0}^{N-1}w(k)$
\WHILE{$|D|<B$}
    \FOR{$k \gets 0$ to $N-1$}
        \STATE Draw a random number $r \sim \mathbb{U}(0, 1)$
        \IF{$r < \hat{w}(k)$}
            \STATE Create a coalition $S$ by randomly selecting $k$ players from $N$
            \STATE Construct a baseline $d$ from $S$ based on $x$ and $x'$ 
            \STATE Add $d$ to $D$
        \ENDIF
    \ENDFOR
\ENDWHILE
\STATE Concatenate $D$ into a tensor with the mini-batch size $B$ 
\STATE Compute mini-batch IG $\mathcal{A} = f_{IG}(x, D, F)$ as Equation{~\ref{Eq_3}}
\STATE Average through mini-batch $\mathcal{V} = \frac{1}{B} \sum_{i=1}^{B} \mathcal{A}_{i}$
\RETURN $\mathcal{V}$
\end{algorithmic} 
\label{algo1}
\end{algorithm}

There are two key points to note in the algorithm. (\romannumeral1) There doesn't have to be a one-to-one mapping from a player to a feature in explained sample $x$. As is shown in Fig.{~\ref{shapley}}($b$), we treat a patch of pixels as a single player to construct baseline sets. (\romannumeral2) The construction of a baseline $d$ of coalition $S$ is by replacing the default values of $x'$ included in $S$ with corresponding values of $x$. The default value is dependent on the task which we will specify in the experiments.

%  method思路，文字感觉写个四分之三页就行了
% 1.图保留a跟c，b的权重说明感觉可以移除到附录里边，或者直接删了如果篇幅太短，就拿来占篇幅。
% 2.把有效联盟跟无效联盟的概念酌情删减了，其实影响不是很大，不然的话，写的太长了。
% 3. 只保留算法1，算法2的划分概念可以讲讲，但是没必要再写一个算法，误差分析的话
% 4. 误差分析得有个公式，我明天翻翻ICLR，貌似有个类似的东西

% 引用图片，公式要统一
% 根据我之前发的数值实验设计文档添加引用
% 哪个实验回答了哪个问题要明确

There are two sources of discrepancies between our SIG and Shapley Value: (\romannumeral1) We only collect a subset of coalitions, with the omission of many other coalitions; (\romannumeral2) The marginal contribution approximated by IG is a deviated from the accurate definition of Shapley Value.

Finally, compared with Shapley Value, the benefits of using SIG are as follows: 
\begin{itemize}
    \item SIG is more suitable in a mini-batch computation setting by leveraging powerful parallel computation devices such as GPUs which is prevalent in current large scale applications.
    \item SIG has the flexibility for users to define a smaller sample size instead of collecting all coalitions.
    \item SIG supports user defined players but can still obtain feature-wise dense estimation of Shapley Values. 
\end{itemize}

% \begin{algorithm}
% \caption{Shapley Integrated Gradients for image}
% \begin{algorithmic}
% \STATE \textbf{INPUT}: image $N$, model $v$
% \STATE \textbf{OUTPUT}: baseline set
% \STATE according to the amount of pixels in the picture, select merge area parameters: m, n
% \STATE divide the picture into m rows and n columns to get m*n areas $L$
% \STATE Treat area as a new player and construct new player set, $M=(L_1,L_2,...,L_{m*n})$
% \FOR{$L_i$ in M}
% \FOR{$i$ in (1,2,...,m*n-1)}
% \STATE For $L_i$, build all coalitions $C=(L_j,...,L_k), C\in B_i$ that contains $i$ players but does not include player $L_i$
% \ENDFOR
% \ENDFOR
% \STATE return baseline sets $B={B_1,B_2,...,B_{m*n-1}}$
% \end{algorithmic} 
% \label{algo2}
% \end{algorithm}

% \textcolor{red}{(What are the user specified hyper-parameters?)}
To summarize, our work identifies that the calculation path of IG takes a shortcut compared to the path of Shapley Value. Consequently, we introduce SIG as a hybrid of Shapley Value and IG which approximates Shapley Value efficiently.

\section{\RomanNumeralCaps{5}. Experiments}
% We propose \czx{SIG} based on new baseline assumption, while we are curious in exploring following questions related to our method's explanation:
% \begin{itemize}
%     \item Can our proposed Shapley Integrated Gradients accurately and reliably approximate the proportion of Shapley Values, as indicated by the Algorithm{~\ref{algo1}}?
%     \item Can our proposed Shapley Integrated Gradients provide a better explanation, as shown in the Algorithm{~\ref{algo2}}, to confirm that our baseline assumption is indeed valid in various practical applications?
% \end{itemize}

%improved 
% In last section, we proposed SIG, a novel baseline construction method from  perspective of Shapley Value. So we wonder following questions:
% In the Methods section, we proposed a new baseline construction method, SIG, from the perspective of Shapley Values. 
In this section, we aim to further validate the soundness of SIG by exploring the answers to the following questions. 
\begin{itemize}
    \item \textbf{Question 1.} Does SIG accurately approximate the proportion of Shapley Values?
    \item \textbf{Question 2.} Does SIG provide improved explanations than IG? If so, is such improvement generic to different data types or instances, and consistent across various kinds of visual input tasks? In short, we wonder how well does SIG outperform IG.
\end{itemize}

\subsection{Experiment Setup}
To answer the aforementioned two critical questions, we devise two sets of experiments. Specially, 
\begin{itemize}
    \item {\bf Simulation of Shapley Value} involves a simple GridWorld environment~\cite{sutton2018reinforcement} comprising 4 distinct sub-tasks to evaluate our SIG's ability of simulating Shapley Value for Question 1;
    \item {\bf Performance of Explanation} focuses on two visual input-based tasks: 1) Expression Code Task utilizing Deviation Learning Network (DLN) model~\cite{zhang2021learning} on a dataset from a private data source.; 2) Image Classification Task employing ResNet model~\cite{he2016deep} on ImageNet Dataset~\cite{deng2009imagenet} for Question 2.
\end{itemize}  

% \jw{\paragraph{Experimental Environment.} including, but not limited to: CPU/GPU, ML libraries/tools, and \textcolor{red}{GitHub repository url (must)}}

\paragraph{Compared Methods.} To objectively assess the performance of explanation provided by SIG, we choose three commonly used baseline methods for comparison, i.e., (\romannumeral1) {\bf random baseline}: baseline sample is randomly selected from a set of samples being explained;
(\romannumeral2) {\bf zero baseline}: baseline value of each feature is set to zero; (\romannumeral3) {\bf mean baseline}: baseline value of each feature is determined as the average value across a set of samples being explained.

\paragraph{Hyperparameters.} For \textit{Simulation of Shapley Value} experiment,  we list several hyperparameters as follows.
\begin{itemize}
    \item[$\blacktriangleright$] \textbf{Q}: represents the ratio of sampled coalitions to all coalitions;
    \item[$\blacktriangleright$] \textbf{N}: denotes the number of baselines sampled from randomly selected coalitions.
\end{itemize}
For \textit{Performance of Explanation} experiment, hyperparameters include: 
\begin{itemize}
    \item[$\blacktriangleright$] \textbf{\{M, N\}}: indicates that an $M \times N$ pixel area is treated as a new player.
\end{itemize}
For a more detailed discussion on the impact of these hyperparameters, please refer to the Appendix.

\subsection{Simulation of Shapley Value}
The simplicity of GridWorld facilitates the computation of Shapley Values, making it an ideal simulation environment to answer Question 1. In GridWorld, the agent performs valid actions (that do not result in out-of-bounds) based on its state to maximize the reward. We take time step $(s,a)$ as a player and treat reward as utility. We construct a 2 $\times$ 2 GridWorld with 12 players under 2 distinct Shapley Value conditions, and a 2 $\times$ 3 GridWorld with 20 players and 2 different Shapley Value conditions. 
The Shapley Values for these four conditions are detailed in the Appendix.

% \usepackage{multirow}
% \usepackage{graphicx}
% \begin{table}[htb!]
% \centering
% \tiny
% \renewcommand\arraystretch{1.1}
% \resizebox{0.5\columnwidth}{!}{%
% \begin{tabular}{c|l|l|l}
% \hline \hline
% \multicolumn{1}{l|}{\textbf{Q}} & \multicolumn{1}{c|}{30\%} & 40\% & 50\% \\ \hline
% \multirow{5}{*}{\textbf{N}}     & 50                        & 50   & 50   \\
%                        & 100                       & 100  & 100  \\
%                        & 150                       & 150  & 150  \\
%                        & 200                       & 200  & 200  \\
%                        & 300                       & 300  & 300  \\ \hline \hline
% \end{tabular}%
% }
% \caption{Hyperparameter settings in 2 $\times$ 2 GridWorld. Each column indicates the number of samples ($\textbf{N}$) selected by SIG based on different sampling ratios ($\textbf{Q}$).}
% \label{hyper2x2}
% \end{table}

\begin{figure}[htp!]
\centering
\setlength{\abovecaptionskip}{0.cm}
\subfigure[]{
    \includegraphics[width=0.45\columnwidth]{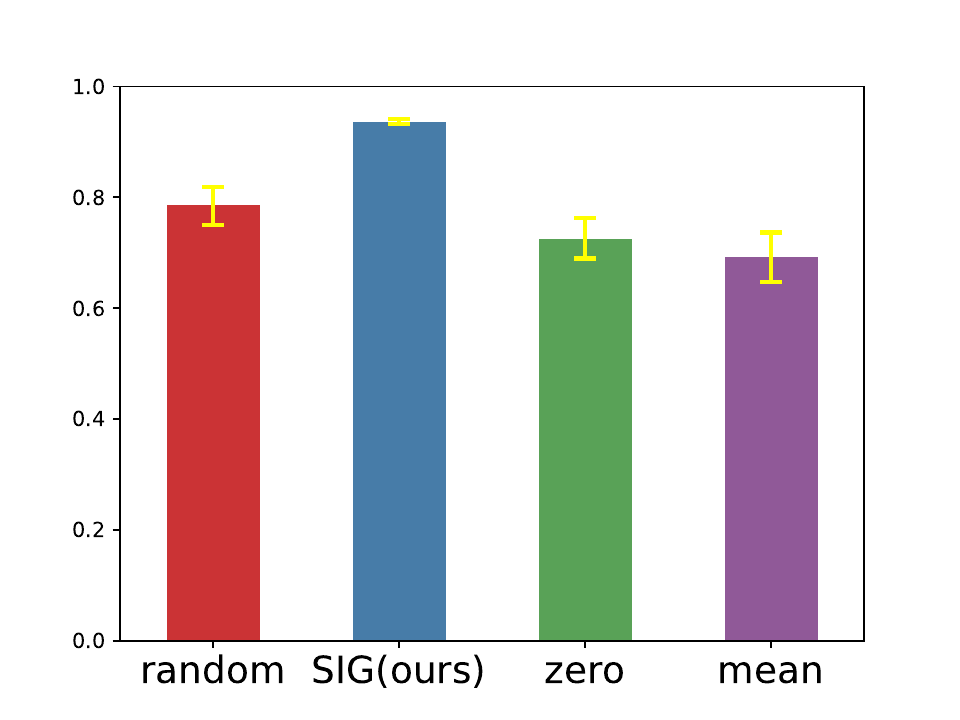}
           \label{grid:1}
           }
\subfigure[]{
    \includegraphics[width=0.45\columnwidth]{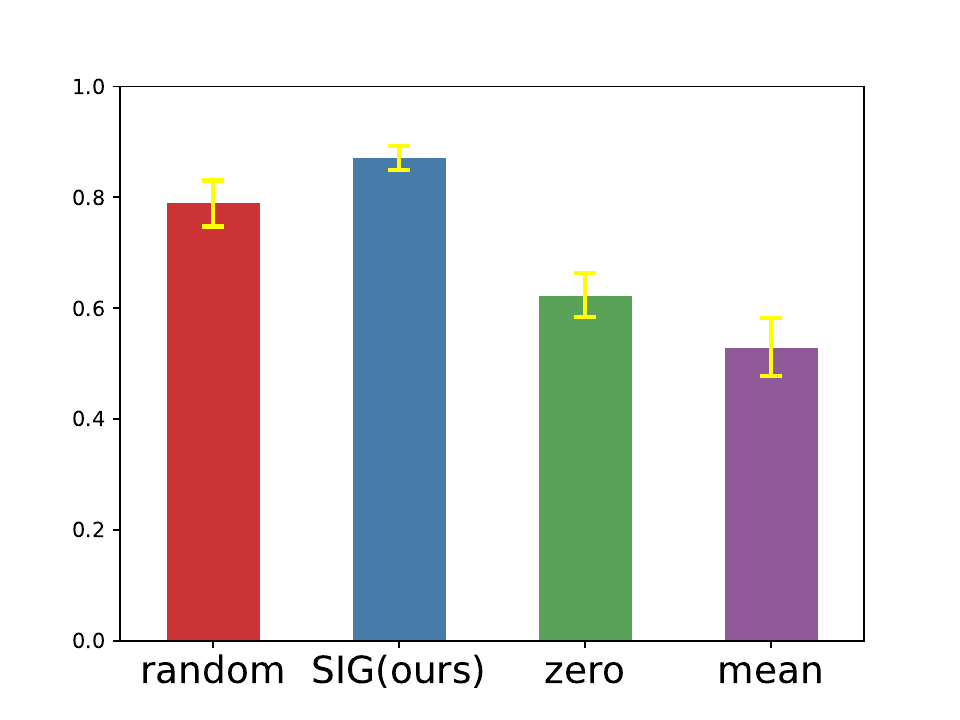}
           \label{grid:2}
           }
\subfigure[]{
    \includegraphics[width=0.45\columnwidth]{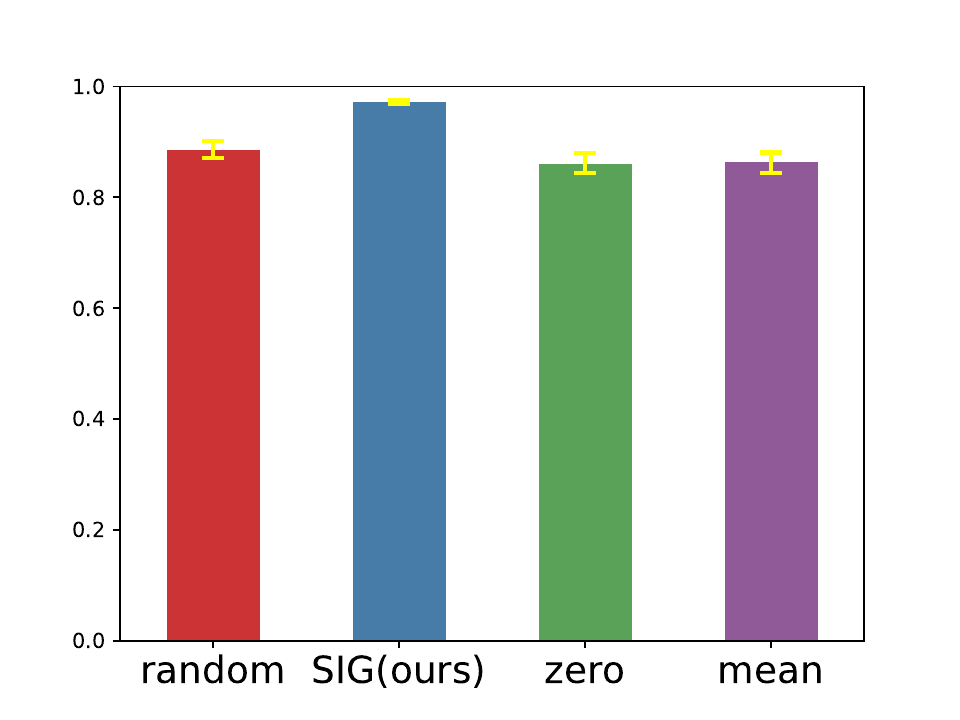}
           \label{grid:3}
           }
\subfigure[]{
    \includegraphics[width=0.45\columnwidth]{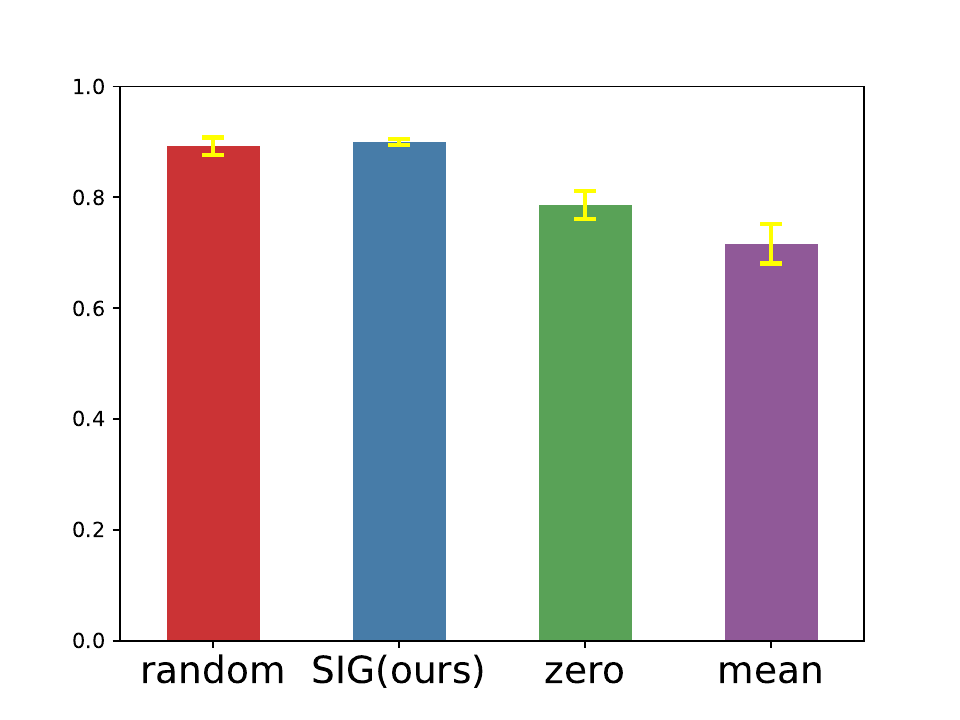}
           \label{grid:4}
           }
\caption{Average Spearmanr metric of four baseline methods in GridWorld. $(a)$ and $(b)$ depict the results for 2 $\times$ 2 GridWorld, and $(c)$ and $(d)$ are the results for 2 $\times$ 3 GridWorld. The yellow line indicates the variance of the Spearmanr metric.}
\label{gridworld}
\setlength{\abovecaptionskip}{-1pt}
% \vspace{-0.3cm}
\end{figure}

\begin{figure}[htp!]
\centering
\subfigure[]{
    \includegraphics[width=0.46\columnwidth]{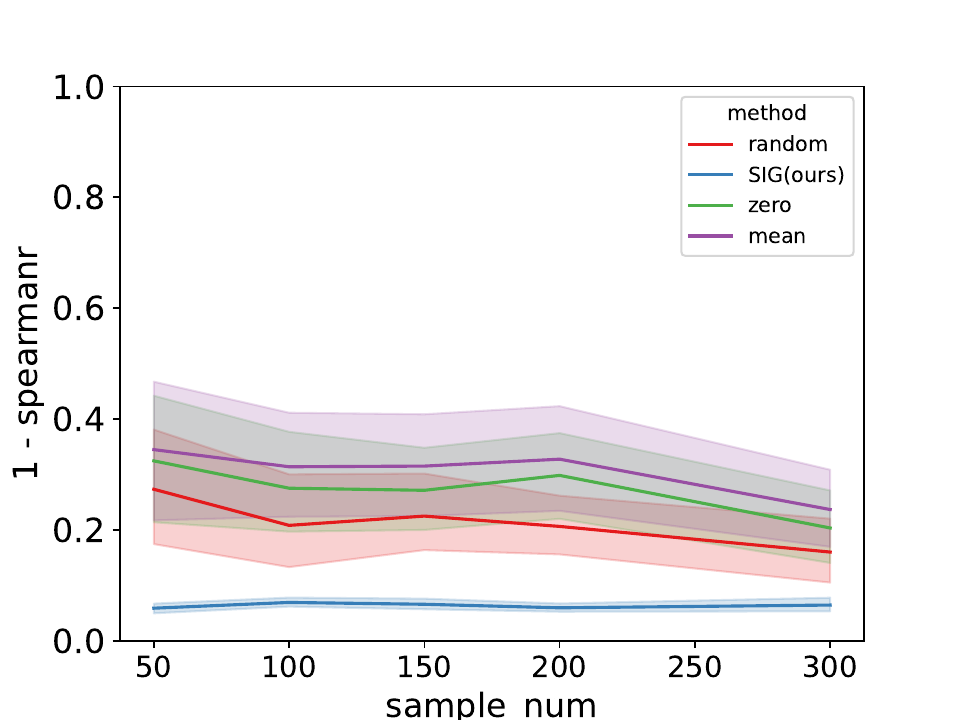}
           \label{hyper1}
           }
\subfigure[]{
    \includegraphics[width=0.46\columnwidth]{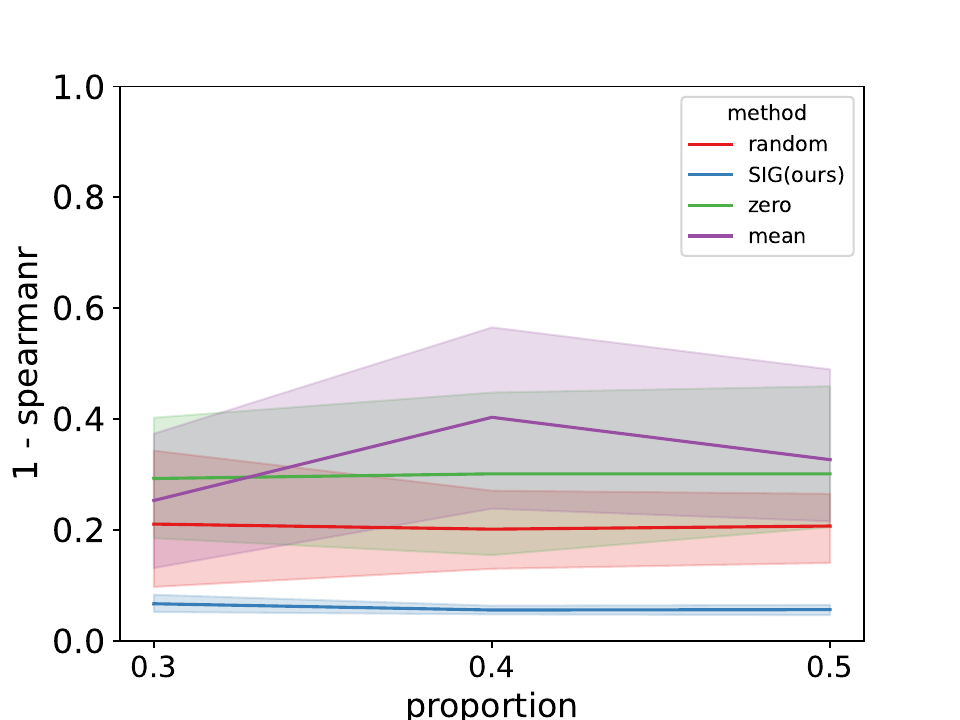}
           \label{hyper2}
           }
%\setlength{\abovecaptionskip}{-1pt}
% \vspace{-0.3cm}
\caption{Average 1-Spearmanr metric of four baseline methods in GridWorld. $(a), (b)$ denote the impacts of hyterparameters $N$ with fixed $Q$ $40\%$ and $Q$ with fixed $N$ 200 respectively in a 2$\times 2$ GridWorld. (Details are documented in the Appendix)}
\label{hyper}
\end{figure}

As a quantitative measurement for the order of player's contribution, we select {\it Spearmanr metric} \cite{gauthier2001detecting} to evaluate the similarity between the computed Shapley Value order and the actual Shapley Value order. 
To investigate the robustness of SIG under varying Shapley Values, 
% To explore the robustness of SIG against hyperparameters, we explore parameter
we test against a set of hyperparameters. Ten independent experiments were conducted to minimize the effect of randomness.
The results are shown in Fig.{~\ref{gridworld}}. 
SIG has a higher average Spearmanr metric 
than the other three baseline methods across different Shapley Value conditions.
It means that order computed by SIG is closer to the player's actual Shapley Value order. Therefore, we conclude a \textbf{positive response} to \textbf{Question 1}.
We also explore the robustness of SIG against hyperparameters. 
% Table~\ref{hyper2x2} lists our hyperparameter settings in the $2 \times 2$ GridWorld, where \textbf{Q} is set to 30\%, 40\%, and 50\%, respectively. For each \textbf{Q}, we explore the performance of the algorithm under the parameters \textbf{N}, 50, 100, 150, 200, and 300 respectively. 
Fig.{~\ref{hyper}} illustrates the performance of the four methods against different hyperparameter settings in a $2 \times 2$ GridWorld. It's obvious that SIG is sensitive to neither hyperparameters $Q$ nor $N$, which shows robustness of SIG. In addition, Fig.{~\ref{gridworld}} demonstrates that the variance of SIG is lower than other three baseline methods, which is another proof of SIG's robustness.
% Interestingly, we found that the randomized baseline sometimes gave similar results to ours when performing the second scoring. The reason for this may be that the coalition chosen for the randomized baseline was selected from a sampling coalition of Shapley Values and contributed in partially simulating the Shapley Values.

\subsection{Performance of Explanation}
In this experiment, we seek for an answer to Question 2 by employing human intuition and quantitative metrics to evaluate algorithmic performance. Specifically, in the Expression Code Task, since the dataset provided by the private data source is not sufficient for quantitative analysis, we measure the performance of methods according to human intuition. In the Image Classification Task, we leverage both human intuition and quantitative metrics.

\begin{figure}[htp!]
 \center
\includegraphics[width=\columnwidth]{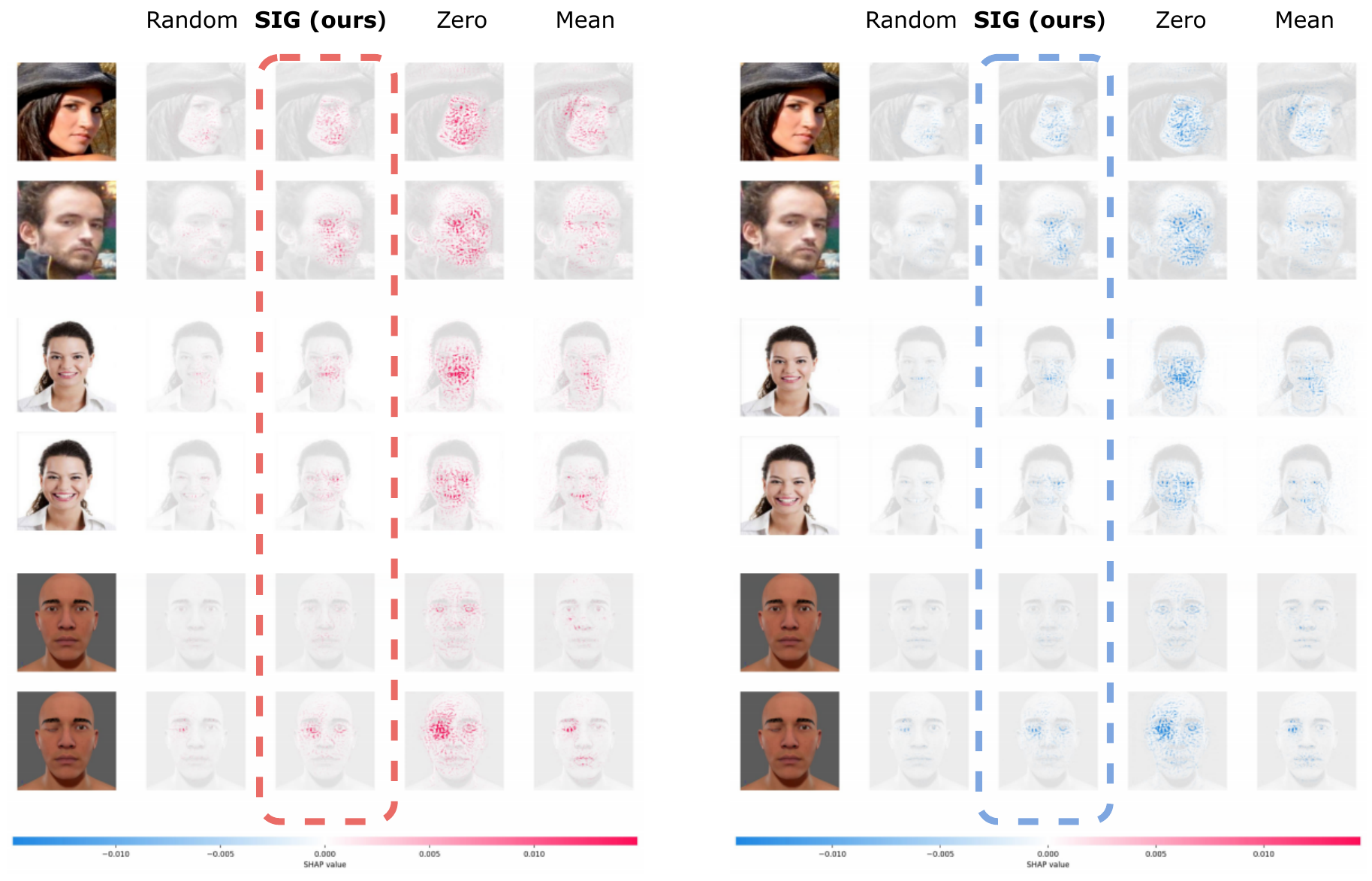}
    \caption{Saliency map of four baseline methods in Expression Code Task where red pixels indicate positive values, while blue pixels denote negative values.}
\label{expression_compare}
% \vspace{-0.3cm}
\end{figure} 

\paragraph{Expression Code Task} 
% In this task, we assume that each pixel in the image is regarded as a player, and iterate over all possible coalitions, then this means that there are $2^{244\times 244}- 2$ coalitions, such a computational cost becomes almost insurmountable.
To reduce computational overhead, we treat an 80 $\times$ 80 pixel block as a single player and use SIG to construct the baseline set. To further clarify the explanation, we divide the pixels into positive and negative pixels. Positive pixels help to reduce the distance between two images ($\text{values} > 0$), while negative pixels help to widen this gap ($\text{values} < 0$).
% Therefore, we treat an 80 $\times$ 80 pixel block as a single player and build the baseline set using our SIG. To further elucidate the explanation, we divide the pixels into positive and negative pixels. Positive pixels assist in narrowing the distance between two images ($\text{values} > 0$), while negative pixels serve to widen this gap ($\text{values} < 0$).
% Result is shown in Fig.{~\ref{expression_compare}}.

As is shown in Fig.{~\ref{expression_compare}}, 
compared with the other three counterparts, SIG method is able to pay more attention to human facial features, such as eyes and lips, which is in good agreement with the first impression when human observe faces. While both zero baseline and mean baseline methods can identify these facial features, they often shift the focus to other parts of the face, even regions outside the face boundary. From human intuition perspective, SIG conveys more explainable interpretations compared to the other three baseline methods. 
% Based on this experimental result, we can first provide a positive answer to Question 2.
% We can say \textbf{`yes'} to improved explanation question in \textbf{Question 2}.

% our SIG method is able to focus more on human facial features, such as eyes and lips, than the other three comparison methods, a property that is consistent with human intuition. While the zero baseline and mean baseline methods are also able to recognize these facial features, they also divert attention to other facial regions or even focus on regions outside the face. Interestingly, we discover that distribution of positive pixels is similar to that of negative pixels, which may be caused by the structure of the convolutional neural network (CNN). 
% We are also pleasantly surprised to find that the computational time required by our SIG method was comparable to other baseline methods, as detailed in the Appendix.

 \begin{figure}[htb!]
 \center
\includegraphics[width=\columnwidth]{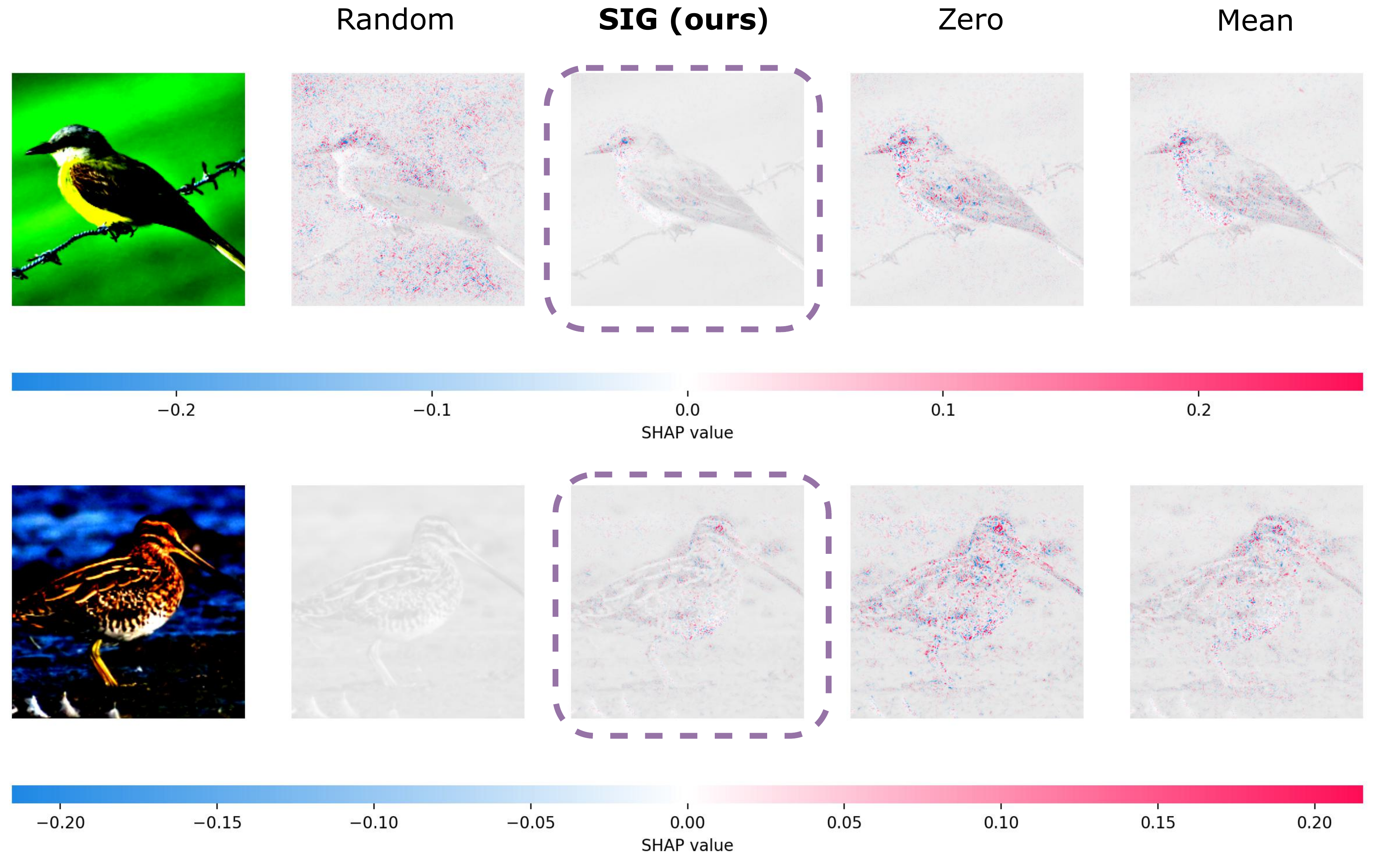}
    \caption{Saliency map of four baseline methods in Image Classification Task. SIG method predominantly focuses on bird, whereas other baseline methods divert their attention to areas outside of the bird.}
\label{imagenet_explanation}
% \vspace{-0.3cm}
\end{figure} 

\begin{figure}[htb!]
\centering
\includegraphics[width=0.8\columnwidth]{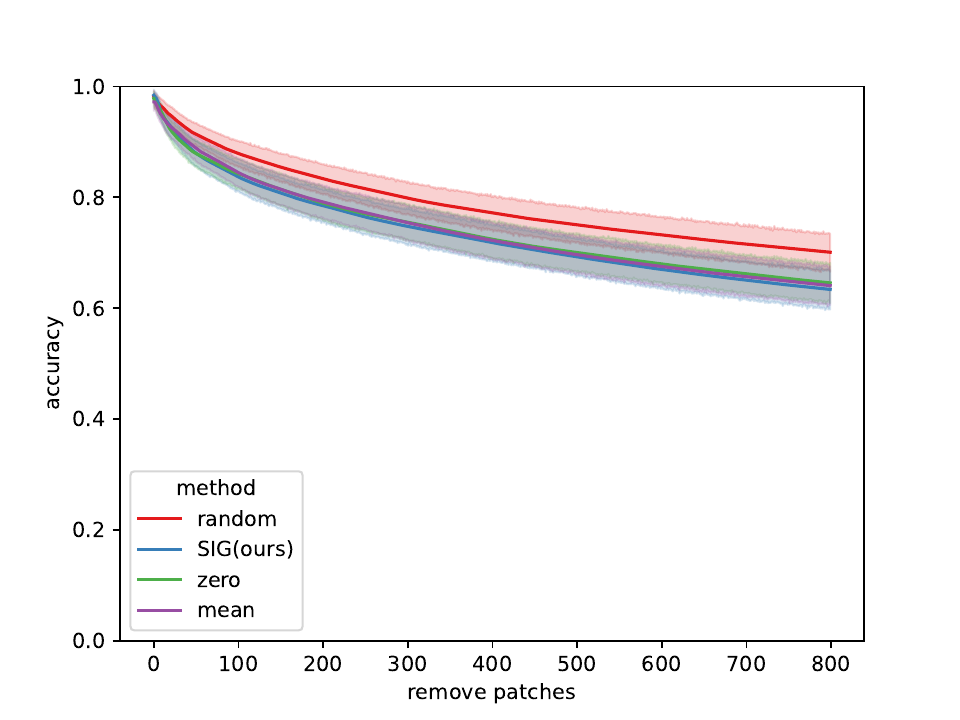} % Reduce the figure size so that it is slightly narrower than the column. Don't use precise values for figure width. This setup will avoid overfull boxes.
\caption{iAccuracy metric plots of four baseline methods in Image Classification Task where $x$-axis signifies the number of pixels removed, while $y$-axis denotes the iAccuracy metric. We present the average value derived from 500 images, along with its 0.95 confidence interval.}
\label{image_quantitative}
% \vspace{-0.3cm}
\end{figure}
\paragraph{Image Classification Task}
% This experiment is to further demonstrate the consistency of SIG's explanation ability on different image input tasks. 
In this experiment, the baseline method generates feature contributions for each output neuron of the model (i.e., each class). We choose the class with the highest probability to reduce computational cost. The rest of the configuration is in line with that in the Expression Code Task. 

We consider {\it iAccuracy metric} adapted from \citet{li2020quantitative} as a quantitative metric to highlight the impact of features on model predictions. The iAccuracy metric specifically means that we sort the contribution scores from high to low and gradually remove pixels based on the sorted results to observe the accuracy of the model predictions. The iAccuracy metric is defined as follows.
\begin{align*}
iAcc(L)=\frac{1}{L+1}(\sum_{k=0}^{L}\mathbf{1}_{F(x^0)=F(x^k)})
\end{align*}
where $x^k$ is noted as the image $x^0$ removing the top $k$ pixels attributed by the baseline methods. We carried out experiments on 500 images. An extended discussion about more images is presented in the Appendix.
As illustrated in Fig.{~\ref{imagenet_explanation}}, SIG emphasizes the silhouette of the bird, which aligns with human instinct. While the zero and mean baselines can also pinpoint the bird's silhouette, they additionally highlight areas outside the silhouette. Furthermore, as depicted in Fig.{~\ref{image_quantitative}}, the impact of removing top features selected by our SIG is close to that of the mean and zero baselines.\footnote{\citet{molnar2020interpretable} argues that the difference of the predicted value after removing the feature may not fully depict the Shapley Value.} 
From the experimental results performed on a broad spectrum of image classification tasks, we find that SIG exhibits genericness to data types or instances, and maintains consistency in terms of explainability across a wide scope of tasks.  
Based on the intuitive image interpretation performance and quantitative analysis in the above two tasks, we can respond with a \textbf{`yes'} to \textbf{Question 2}.

Overall, by exploring the answers to Question 1 and Question 2, we experimentally verified the abilities of SIG: improved explainability, generic to data types or instances, and consistent interpretability across application domains. 

\section{\RomanNumeralCaps{6}. Conclusion and Discussion}
% \textbf{Hyper-parameter}
In this work, we rethink baseline of Integrated Gradients (IG) from the perspective of Shapley Value. We observe and theoretically analyse that IG can be viewed as Auman-Shapley Value under certain assumptions. Specifically, a set of baseline aligns with the coalitions in Shapley Value, thus tackleing the challenges of exploiting single baselines.
% it takes a shortcut compared to Shapley Value. This may lead to inaccuracy feature's contribution. And it's hard to construct baseline for IG.
% the baseline construction of the Integrated Gradient from the perspective of Shapley Value. We show that in some cases the Integrated Gradient can be viewed as an Aumann-Shapley Value and that it essentially takes a shortcut compared to Shapley Value. 
Therefore, we propose a novel baseline construction method SIG, which is a hybrid of the IG and Shapley Value that provides an improved, generic,  and consistent explanation compared to existing baseline methods for IG. The time complexity is mainly dominated two iterations of sampling from coalitions, leading to $\mathcal{O}(n^2)$.
% Experimental results demonstrate that SIG approximates to Shapley Value and provide explanations more align with human intuition and similar influence with common used baselines.

Nevertheless, there remains intriguing challenges that deserve further research. For example, our current creation of players using patches of pixels method is slightly coarse \cite{ren2021towards}; the method for fitting Shapley Value proportions exhibits considerable randomness \cite{ando2020monte}; timing performance can be further enhanced \cite{chen2023algorithms}. Addressing these issues is an important goal for our ongoing and future research. 

\renewcommand*{\thesection}{\Alph{section}}
\renewcommand*{\thesubsection}{\thesection.\arabic{subsection}.}

\bibliography{aaai24}

\begin{thebibliography}{46}
\providecommand{\natexlab}[1]{#1}

\bibitem[{Aas, Jullum, and L{\o}land(2021)}]{aas2021explaining}
Aas, K.; Jullum, M.; and L{\o}land, A. 2021.
\newblock Explaining individual predictions when features are dependent: More
  accurate approximations to Shapley values.
\newblock \emph{Artificial Intelligence}, 298: 103502.

\bibitem[{Adadi and Berrada(2018)}]{adadi2018peeking}
Adadi, A.; and Berrada, M. 2018.
\newblock Peeking inside the black-box: a survey on explainable artificial
  intelligence (XAI).
\newblock \emph{IEEE access}, 6: 52138--52160.

\bibitem[{Ancona et~al.(2017)Ancona, Ceolini, {\"O}ztireli, and
  Gross}]{ancona2017towards}
Ancona, M.; Ceolini, E.; {\"O}ztireli, C.; and Gross, M. 2017.
\newblock Towards better understanding of gradient-based attribution methods
  for deep neural networks.
\newblock \emph{arXiv preprint arXiv:1711.06104}.

\bibitem[{Ancona, Oztireli, and Gross(2019)}]{ancona2019explaining}
Ancona, M.; Oztireli, C.; and Gross, M. 2019.
\newblock Explaining deep neural networks with a polynomial time algorithm for
  shapley value approximation.
\newblock In \emph{International Conference on Machine Learning}, 272--281.
  PMLR.

\bibitem[{Ando and Takase(2020)}]{ando2020monte}
Ando, K.; and Takase, K. 2020.
\newblock Monte Carlo algorithm for calculating the Shapley values of minimum
  cost spanning tree games.
\newblock \emph{Journal of the Operations Research Society of Japan}, 63(1):
  31--40.

\bibitem[{Angelov and Soares(2020)}]{angelov2020towards}
Angelov, P.; and Soares, E. 2020.
\newblock Towards explainable deep neural networks (xDNN).
\newblock \emph{Neural Networks}, 130: 185--194.

\bibitem[{Arrieta et~al.(2020)Arrieta, D{\'\i}az-Rodr{\'\i}guez, Del~Ser,
  Bennetot, Tabik, Barbado, Garc{\'\i}a, Gil-L{\'o}pez, Molina, Benjamins
  et~al.}]{arrieta2020explainable}
Arrieta, A.~B.; D{\'\i}az-Rodr{\'\i}guez, N.; Del~Ser, J.; Bennetot, A.; Tabik,
  S.; Barbado, A.; Garc{\'\i}a, S.; Gil-L{\'o}pez, S.; Molina, D.; Benjamins,
  R.; et~al. 2020.
\newblock Explainable Artificial Intelligence (XAI): Concepts, taxonomies,
  opportunities and challenges toward responsible AI.
\newblock \emph{Information fusion}, 58: 82--115.

\bibitem[{Aumann and Shapley(1974)}]{aumann1974values}
Aumann, R.~J.; and Shapley, L.~S. 1974.
\newblock \emph{Values of Non-Atomic Games}.
\newblock Princeton University Press.

\bibitem[{Aumann and Shapley(2015)}]{aumann2015values}
Aumann, R.~J.; and Shapley, L.~S. 2015.
\newblock \emph{Values of non-atomic games}.
\newblock Princeton University Press.

\bibitem[{Binder et~al.(2016)Binder, Montavon, Lapuschkin, M{\"u}ller, and
  Samek}]{binder2016layer}
Binder, A.; Montavon, G.; Lapuschkin, S.; M{\"u}ller, K.-R.; and Samek, W.
  2016.
\newblock Layer-wise relevance propagation for neural networks with local
  renormalization layers.
\newblock In \emph{International Conference on Artificial Neural Networks},
  63--71. Springer.

\bibitem[{Chen et~al.(2023)Chen, Covert, Lundberg, and
  Lee}]{chen2023algorithms}
Chen, H.; Covert, I.~C.; Lundberg, S.~M.; and Lee, S.-I. 2023.
\newblock Algorithms to estimate Shapley value feature attributions.
\newblock \emph{Nature Machine Intelligence}, 1--12.

\bibitem[{Chen, Lundberg, and Lee(2021)}]{chen2021explaining}
Chen, H.; Lundberg, S.; and Lee, S.-I. 2021.
\newblock Explaining models by propagating Shapley values of local components.
\newblock In \emph{Explainable AI in Healthcare and Medicine}, 261--270.
  Springer.

\bibitem[{Chen, Lundberg, and Lee(2022)}]{chen2022explaining}
Chen, H.; Lundberg, S.~M.; and Lee, S.-I. 2022.
\newblock Explaining a series of models by propagating Shapley values.
\newblock \emph{Nature communications}, 13(1): 4512.

\bibitem[{Dabkowski and Gal(2017)}]{dabkowski2017real}
Dabkowski, P.; and Gal, Y. 2017.
\newblock Real time image saliency for black box classifiers.
\newblock \emph{Advances in neural information processing systems}, 30.

\bibitem[{Das and Rad(2020)}]{das2020opportunities}
Das, A.; and Rad, P. 2020.
\newblock Opportunities and challenges in explainable artificial intelligence
  (xai): A survey.
\newblock \emph{arXiv preprint arXiv:2006.11371}.

\bibitem[{Deng et~al.(2009)Deng, Dong, Socher, Li, Li, and
  Fei-Fei}]{deng2009imagenet}
Deng, J.; Dong, W.; Socher, R.; Li, L.-J.; Li, K.; and Fei-Fei, L. 2009.
\newblock Imagenet: A large-scale hierarchical image database.
\newblock In \emph{2009 IEEE conference on computer vision and pattern
  recognition}, 248--255. Ieee.

\bibitem[{Duell et~al.(2023)Duell, Fan, Fu, and Seisenberger}]{duell2023batch}
Duell, J.; Fan, X.; Fu, H.; and Seisenberger, M. 2023.
\newblock Batch Integrated Gradients: Explanations for Temporal Electronic
  Health Records.
\newblock In \emph{International Conference on Artificial Intelligence in
  Medicine}, 120--124. Springer.

\bibitem[{Enguehard(2023)}]{enguehard2023sequential}
Enguehard, J. 2023.
\newblock Sequential Integrated Gradients: a simple but effective method for
  explaining language models.
\newblock \emph{arXiv preprint arXiv:2305.15853}.

\bibitem[{Feng et~al.(2022)Feng, Zhou, Tarun, and Nair}]{feng2022comparing}
Feng, T.; Zhou, Z.; Tarun, J.; and Nair, V.~N. 2022.
\newblock Comparing Baseline Shapley and Integrated Gradients for Local
  Explanation: Some Additional Insights.
\newblock \emph{arXiv preprint arXiv:2208.06096}.

\bibitem[{Frye et~al.(2020)Frye, de~Mijolla, Begley, Cowton, Stanley, and
  Feige}]{frye2020shapley}
Frye, C.; de~Mijolla, D.; Begley, T.; Cowton, L.; Stanley, M.; and Feige, I.
  2020.
\newblock Shapley explainability on the data manifold.
\newblock \emph{arXiv preprint arXiv:2006.01272}.

\bibitem[{Gauthier(2001)}]{gauthier2001detecting}
Gauthier, T.~D. 2001.
\newblock Detecting trends using Spearman's rank correlation coefficient.
\newblock \emph{Environmental forensics}, 2(4): 359--362.

\bibitem[{Ghorbani, Abid, and Zou(2019)}]{ghorbani2019interpretation}
Ghorbani, A.; Abid, A.; and Zou, J. 2019.
\newblock Interpretation of neural networks is fragile.
\newblock In \emph{Proceedings of the AAAI conference on artificial
  intelligence}, volume~33, 3681--3688.

\bibitem[{Ghorbani and Zou(2019)}]{ghorbani2019data}
Ghorbani, A.; and Zou, J. 2019.
\newblock Data shapley: Equitable valuation of data for machine learning.
\newblock In \emph{International Conference on Machine Learning}, 2242--2251.
  PMLR.

\bibitem[{Gunning and Aha(2019)}]{gunning2019darpa}
Gunning, D.; and Aha, D. 2019.
\newblock DARPA’s explainable artificial intelligence (XAI) program.
\newblock \emph{AI magazine}, 40(2): 44--58.

\bibitem[{He et~al.(2016)He, Zhang, Ren, and Sun}]{he2016deep}
He, K.; Zhang, X.; Ren, S.; and Sun, J. 2016.
\newblock Deep residual learning for image recognition.
\newblock In \emph{Proceedings of the IEEE conference on computer vision and
  pattern recognition}, 770--778.

\bibitem[{Heskes et~al.(2020)Heskes, Sijben, Bucur, and
  Claassen}]{heskes2020causal}
Heskes, T.; Sijben, E.; Bucur, I.~G.; and Claassen, T. 2020.
\newblock Causal shapley values: Exploiting causal knowledge to explain
  individual predictions of complex models.
\newblock \emph{Advances in neural information processing systems}, 33:
  4778--4789.

\bibitem[{Jha et~al.(2020)Jha, K~Aicher, R~Gazzara, Singh, and
  Barash}]{jha2020enhanced}
Jha, A.; K~Aicher, J.; R~Gazzara, M.; Singh, D.; and Barash, Y. 2020.
\newblock Enhanced integrated gradients: improving interpretability of deep
  learning models using splicing codes as a case study.
\newblock \emph{Genome biology}, 21(1): 1--22.

\bibitem[{Kumar et~al.(2020)Kumar, Venkatasubramanian, Scheidegger, and
  Friedler}]{kumar2020problems}
Kumar, I.~E.; Venkatasubramanian, S.; Scheidegger, C.; and Friedler, S. 2020.
\newblock Problems with Shapley-value-based explanations as feature importance
  measures.
\newblock In \emph{International Conference on Machine Learning}, 5491--5500.
  PMLR.

\bibitem[{Li et~al.(2020)Li, Shi, Li, Bai, Song, Cao, and
  Chen}]{li2020quantitative}
Li, X.-H.; Shi, Y.; Li, H.; Bai, W.; Song, Y.; Cao, C.~C.; and Chen, L. 2020.
\newblock Quantitative evaluations on saliency methods: An experimental study.
\newblock \emph{arXiv preprint arXiv:2012.15616}.

\bibitem[{Lundberg and Lee(2017)}]{lundberg2017unified}
Lundberg, S.~M.; and Lee, S.-I. 2017.
\newblock A unified approach to interpreting model predictions.
\newblock \emph{Advances in neural information processing systems}, 30.

\bibitem[{Merrick and Taly(2020)}]{merrick2020explanation}
Merrick, L.; and Taly, A. 2020.
\newblock The explanation game: Explaining machine learning models using
  shapley values.
\newblock In \emph{International Cross-Domain Conference for Machine Learning
  and Knowledge Extraction}, 17--38. Springer.

\bibitem[{Mitchell et~al.(2022)Mitchell, Cooper, Frank, and
  Holmes}]{mitchell2022sampling}
Mitchell, R.; Cooper, J.; Frank, E.; and Holmes, G. 2022.
\newblock Sampling permutations for shapley value estimation.
\newblock \emph{The Journal of Machine Learning Research}, 23(1): 2082--2127.

\bibitem[{Molnar(2020)}]{molnar2020interpretable}
Molnar, C. 2020.
\newblock \emph{Interpretable machine learning}.
\newblock Lulu. com.

\bibitem[{Mudrakarta et~al.(2018)Mudrakarta, Taly, Sundararajan, and
  Dhamdhere}]{mudrakarta2018did}
Mudrakarta, P.~K.; Taly, A.; Sundararajan, M.; and Dhamdhere, K. 2018.
\newblock Did the model understand the question?
\newblock \emph{arXiv preprint arXiv:1805.05492}.

\bibitem[{Pourdarbani et~al.(2023)Pourdarbani, Sabzi, Nadimi, and
  Paliwal}]{pourdarbani2023interpretation}
Pourdarbani, R.; Sabzi, S.; Nadimi, M.; and Paliwal, J. 2023.
\newblock Interpretation of Hyperspectral Images Using Integrated Gradients to
  Detect Bruising in Lemons.
\newblock \emph{Horticulturae}, 9(7): 750.

\bibitem[{Ras et~al.(2022)Ras, Xie, Van~Gerven, and Doran}]{ras2022explainable}
Ras, G.; Xie, N.; Van~Gerven, M.; and Doran, D. 2022.
\newblock Explainable deep learning: A field guide for the uninitiated.
\newblock \emph{Journal of Artificial Intelligence Research}, 73: 329--396.

\bibitem[{Ren et~al.(2021)Ren, Zhou, Chen, and Zhang}]{ren2021towards}
Ren, J.; Zhou, Z.; Chen, Q.; and Zhang, Q. 2021.
\newblock Towards a Game-Theoretic View of Baseline Values in the Shapley
  Value.

\bibitem[{Shapley(1951)}]{shapley1951notes}
Shapley, L.~S. 1951.
\newblock Notes on the n-Person Game—II: The Value of an n-Person
  Game.(1951).
\newblock \emph{Lloyd S Shapley}.

\bibitem[{Shrikumar, Greenside, and Kundaje(2017)}]{shrikumar2017learning}
Shrikumar, A.; Greenside, P.; and Kundaje, A. 2017.
\newblock Learning important features through propagating activation
  differences.
\newblock In \emph{International conference on machine learning}, 3145--3153.
  PMLR.

\bibitem[{Simonyan, Vedaldi, and Zisserman(2013)}]{simonyan2013deep}
Simonyan, K.; Vedaldi, A.; and Zisserman, A. 2013.
\newblock Deep inside convolutional networks: Visualising image classification
  models and saliency maps.
\newblock \emph{arXiv preprint arXiv:1312.6034}.

\bibitem[{Sundararajan and Najmi(2020)}]{sundararajan2020many}
Sundararajan, M.; and Najmi, A. 2020.
\newblock The many Shapley values for model explanation.
\newblock In \emph{International conference on machine learning}, 9269--9278.
  PMLR.

\bibitem[{Sundararajan, Taly, and Yan(2017)}]{sundararajan2017axiomatic}
Sundararajan, M.; Taly, A.; and Yan, Q. 2017.
\newblock Axiomatic attribution for deep networks.
\newblock In \emph{International conference on machine learning}, 3319--3328.
  PMLR.

\bibitem[{Sutton and Barto(2018)}]{sutton2018reinforcement}
Sutton, R.~S.; and Barto, A.~G. 2018.
\newblock Reinforcement learning: An introduction.
\newblock chapter 3.7.

\bibitem[{Tan(2023)}]{tan2023maximum}
Tan, H. 2023.
\newblock Maximum entropy baseline for integrated gradients.
\newblock In \emph{2023 International Joint Conference on Neural Networks
  (IJCNN)}, 1--8. IEEE.

\bibitem[{Yang, Wang, and Bilgic(2023)}]{yang2023idgi}
Yang, R.; Wang, B.; and Bilgic, M. 2023.
\newblock IDGI: A Framework to Eliminate Explanation Noise from Integrated
  Gradients.
\newblock In \emph{Proceedings of the IEEE/CVF Conference on Computer Vision
  and Pattern Recognition}, 23725--23734.

\bibitem[{Zhang et~al.(2021)Zhang, Ji, Chen, Ding, and Fan}]{zhang2021learning}
Zhang, W.; Ji, X.; Chen, K.; Ding, Y.; and Fan, C. 2021.
\newblock Learning a facial expression embedding disentangled from identity.
\newblock In \emph{Proceedings of the IEEE/CVF conference on computer vision
  and pattern recognition}, 6759--6768.

\end{thebibliography}

\clearpage
\appendix
\section*{Technical Appendix}
% In this technical appendix, we first summarize the detailed notations used in the “Method” section, and further details our proof. Then we describe our experiment design and supplementary experimental results.
\subsection{Details of the Theory}
\paragraph{Notations.} Table{~\ref{tab:notation}} summarizes the detailed notations used in the ``Method'' section.
\begin{table}[h]
\centering
\begin{tabular}{c|c}
\toprule
    Symbol & Meaning  \\
    \midrule
    $x$ & The sample to explain  \\
    $x'$ & The baseline sample for IG \\
    $x_i$ & The ith feature/player of $x$ \\
    $r_1$, $r_2$ & Players of $x'$ \\
    $S_1$, $S_2$ & Players of $x$ \\
    $v()$ & The utility function of SV \\
    $F()$ & The model/function to explain and evaluate \\
    $t$ & A scalar value indicates a proportion \\
    $I$ & A general complete set \\
    $N$ & The set of all players in a game \\
    $S$ & A coalition in a game \\
    $i$ & The indicator of the player of interest \\
    $V_i(S)$ & The marginal contribution of $x_i$ to $S$ \\
    $w_i(S)$ & The weight of $V_i(S)$ \\
    $k$ & A scalar value of the size of $S$ \\
    $\hat{w}(k)$ & The normalized $w(k)$ \\
    $D$ & The random sampled basline set of SIG \\
    $\mathcal{A}$ & A tensor of a mini-batch of IG results \\
    $\mathcal{V}$ & A tensor of the approximated SV result \\
\bottomrule
\end{tabular}
\caption{Notations of Method Section}\label{tab:notation}
\end{table}

\paragraph{Proof.} We prove that our proposed \textbf{proportional sampling} is an unbiased estimator of the true Shapley Value. Let $f^{PS}_{SV}(x_i)$ denote the estimated Shapley Value of $x_i$ computed leveraging proportional sampling. $p_i(S)$ is the probability of $S$ being sampled. $V_i(S)$ is the marginal contribution. $C_k$ is the value of $\binom{|N|-1}{k}$. Then $p_i(S) = \frac{1}{|N|} p_i(k) = \frac{1}{N \times C_k}$. Suppose we sampled for $M$ times $S$ based on probability function $p_i(S)$, 

\begin{align*}
\mathbb{E}(f^{PS}_{SV}(x_i)) &= \frac{1}{M}\sum p_i(S)V_i(S) \\
&= \frac{1}{M}\sum \frac{1}{|N|} \sum_{k=0}^{N-1} p_i(k)V_i(S; |S|=k) \\
&= \frac{1}{M}\sum \frac{1}{|N|} \sum_{k=0}^{N-1} \frac{1}{C_k} \sum_{j=0}^{C_k} V_i(S_j) \\
&= \frac{1}{M}\sum \sum_{S\in N/\{i\}} \frac{1}{|N| \times C_k} V_i(S) \\
&= \frac{1}{M}\sum 1 \cdot \sum_{S\in N/\{i\}} w_i(S)V_i(S) \\ 
&= \mathbb{E}(\sum_{S\in N/\{i\}} w_i(S)V_i(S)) = \mathbb{E}(f_{SV}(x_i))
\end{align*}

\paragraph{Algorithms.} Although we provided the algorithm of general Shapley Intergratd Gradients (SIG) in the main paper. We made small adjustments to it when applying it to game data and image data respectively.
\textbf{Game Data}. Sampling from all coalitions can be time-consuming, and at times, it might even seem infeasible in actual cases. Therefore, we initially randomly sample a proportion Q of coalitions. From this subset, we then sample N baselines. \textbf{Image Data}. Direct computation for each pixel is nearly unfeasible. To simplify the process, we consider an area of $M\times N$ pixels as a single player, thereby streamlining the computation of contributions.

% $\dots$ \textbf{Image Data} $\dots$ \textcolor{red}{Mainly talk about how do the hyperparameters in experiments can be mapped to the algorithm parameters in Algorithm 1: SIG; keep succinct, don't make it complicated.}

\subsection{Experimental Design}

\paragraph{Reward settings in the GridWorld} In the GridWorld environment, two types of reward functions are utilized.  The first reward is denoted by $r_1$, which is -1 for each step taken by the agent before it reaches the end position, and +1 when it reaches the end position. The second reward function $r_2$ works as follows:
\begin{itemize}
    \item For every step taken by the agent before it reaches the termination position for the first time, a -1 reward is given.
    \item On reaching the termination position for the first time, the agent is awarded a +1 reward.
    \item Subsequent to the agent's initial arrival at the termination position:
    \begin{itemize}
        \item -0.5 reward is given for each step if the agent does not reach the termination position.
        \item +1.5 reward is granted if the agent reaches the termination position again.
    \end{itemize}
\end{itemize}
 Shapley Value computes the marginal contribution of players based on combinations of players. In contrast, our reward function addresses permutations. We define the reward function of a combination $S$ to be the maximum reward associated with any permutation $P$ within combination $S$, i.e., $v(S) = \max v(P), P\in S$. These two reward function are employed in both 2 $\times$ 2 GridWorld and 2 $\times$ 3 GrdiWorld, resulting in Shapley Value shown in  Fig.{~\ref{train_effect}}.
% \begin{figure}[htp!]
%  \center
% \includegraphics[width=\columnwidth]{appendix/Shapley.pdf}
%     \caption{Shapley Value over four cases. a and b represents 2 $\times$ 2 GridWorld with reward function $r_1$ and $r_2$ respectively. a and b represents 2 $\times$ 2 GridWorld with reward function $r_1$ and $r_2$ respectively.\textcolor{red}{repetition?? seems duplicated}}
% \label{shapley}
% % \vspace{-0.3cm}
% \end{figure} 

Since reward functions are often not compatible with Deep Neural Networks (DNN). To address this, we utilize trained deep neural networks as a substitute for the reward functions. In detail, we take coalitions with one-hot encoding as samples, and regard output of reward function as label. To evaluate performance of trained DNN model, we replace reward function with DNN in the computation of Shapley Value to assess the DNN's ability to emulate the reward function. 
We present one representative example to demonstrate that the DNN can accurately emulate the Shapley Value. Our trained DNN can seamlessly replace the utility function, producing Shapley Values that are nearly identical. Further details can be found in the accompanying code.

\begin{figure}[htp!]
 \center
\includegraphics[width=\columnwidth]{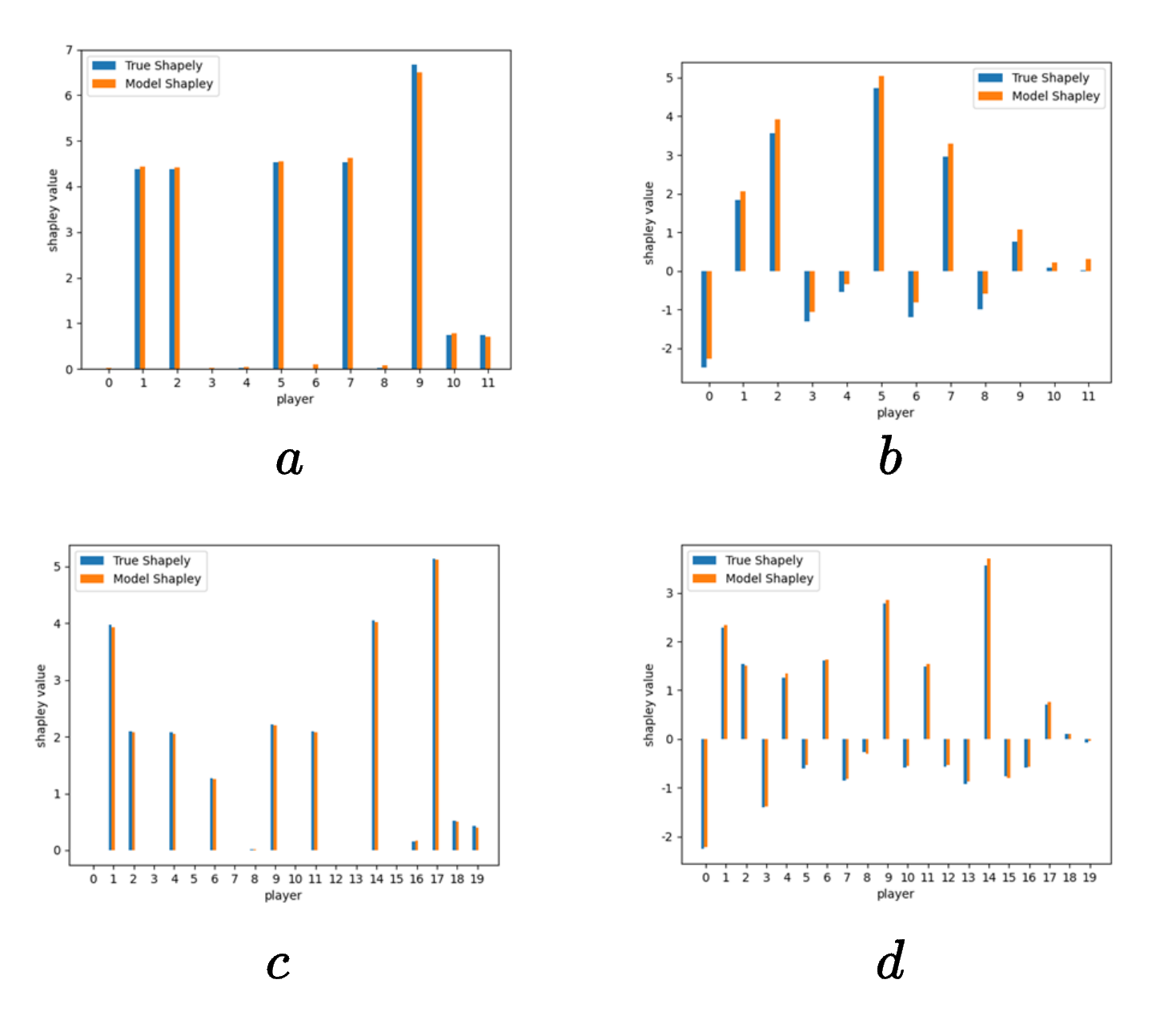}
    \caption{Evaluation of DNN. We replace reward function with DNN in the computation of Shapley Value to assess the DNN's ability to emulate the reward function. Subfigures $a$ and $b$ represent 2 $\times$ 2 GridWorld with reward function $r_1$ and $r_2$ respectively. Subfigures $c$ and $d$ represents 2 $\times$ 3 GridWorld with reward function $r_1$ and $r_2$ respectively.}
\label{train_effect}
% \vspace{-0.3cm}
\end{figure}

\subsection{Models}
\paragraph{GridWorld} 
For the experiments conducted in GridWorld, we employ a fully connected network that comprises one hidden layer with 64 neurons and another hidden layer with 32 neurons.

\paragraph{Expression Code Task} For the experiments conducted in Expression Code Tasks, we utilize DLN model and a private dataset.

\paragraph{Image Classification Task} For the experiments conducted in Image Classification Task, we apply Resnet model built in Pytorch and ImageNet dataset.

\subsection{Supplementary Results}
\paragraph{Running Time} In Image Classification Tasks, we further evaluate the running time of the baseline methods. As depicted in Fig.{~\ref{time_imagenet}}, the running time of our SIG closely mirrors that of the zero and mean baselines, outperforming the random baseline.  Additionally, while the variance in running time for our SIG is slightly greater than that of the zero and mean baselines initially, it reduces as the number of images increases. This can mainly be attributed to ResNet's capability to handle batch operations, which results in similar running times.
\begin{figure}[htp!]
\center
\includegraphics[width=\columnwidth]{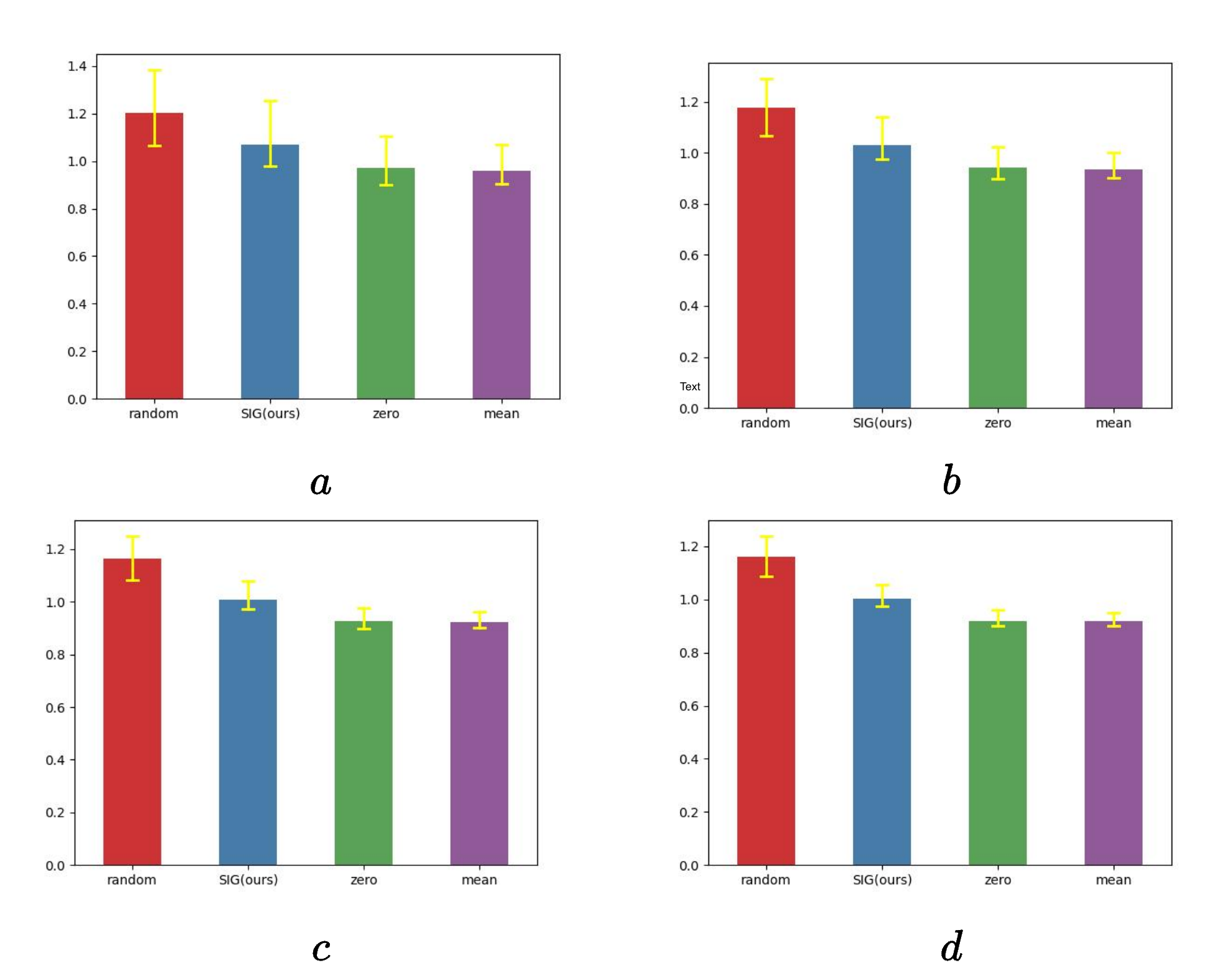}
\caption{ Running time of baseline methods with ResNet model and ImageNet dataset. Subfigures $a$, $b$, $c$ and $d$ represent running time over 300, 500, 800, 1000 images respectively.}
\label{time_imagenet}
% \vspace{-0.3cm}
\end{figure}

While for GridWorld, where the computation of the Shapley Value doesn't support batch operation, there's a significant increase in running time, as reflected in Table{~\ref{tab:runtime}}, our SIG's running time is notably higher than that of the other baselines since our SIG iterates baseline one by one.

\begin{table}[]
\centering
\resizebox{\columnwidth}{!}{%
\begin{tabular}{cc|c|c|c|
>{\columncolor[HTML]{ECF4FF}}c |c|c}
\hline \hline
Environment           & Reward & $Q$ & $N$ & Random & \textbf{SIG (ours)} & Zero & Mean \\ \hline
\multicolumn{1}{c|}{} & $r_1$  &     &     & 4.59s   & \textbf{55.05s}      & 4.59s & 4.59s \\ \cline{2-2} \cline{5-8} 
\multicolumn{1}{c|}{\multirow{-2}{*}{2 $\times$ 2}} & $r_2$ & \multirow{-2}{*}{40\%} & \multirow{-2}{*}{200}  & 4.59s & \textbf{55.07s} & 4.59s & 4.59s \\ \hline \hline
\multicolumn{1}{c|}{} & $r_1$  &     &     & 4.59s  & \textbf{91.79s}     & 4.57s & 4.59s \\ \cline{2-2} \cline{5-8} 
\multicolumn{1}{c|}{\multirow{-2}{*}{2 $\times$ 3}} & $r_2$ & \multirow{-2}{*}{40\%} & \multirow{-2}{*}{1500} & 4.59s & \textbf{91.77s} & 4.59s & 4.59s \\ \hline \hline
\end{tabular}%
}
\caption{Running time of each baseline method. The unit of running time is second (s). We take one example to explain that running time of our SIG is quite longer than other three baselines since GridWorld doesn't support batch operation.}
\label{tab:runtime}
\end{table}

\paragraph{GridWorld} We delve into the effects of hyperparameters across different GridWorld environments. The influence of hypermeters $N$ with fixed $Q$ $40\%$ is depicted in Fig.{~\ref{sample_num}}. Further details can be found in our codes. The outcomes indicate that our SIG displays impressive robustness across various Shapley Value cases for $N$.

\begin{figure}[htp!]
 \center
\includegraphics[width=\columnwidth]{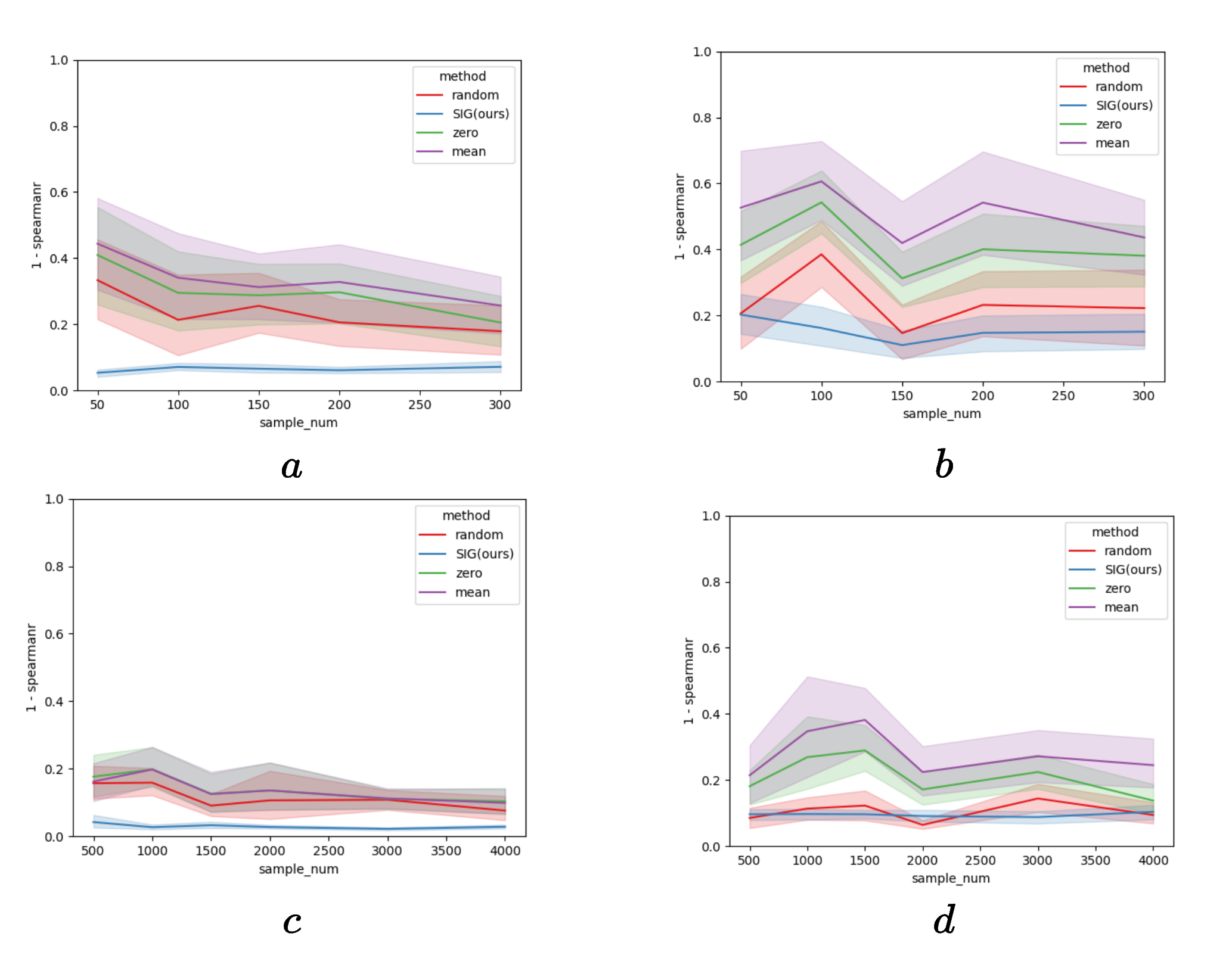}
    \caption{Average 1-Spearmanr metric for hypterparameters $N$ across four GridWorld situations with fixed $Q$ $40\%$. Subfigures $a$ and $b$ represent 2 $\times$ 2 GridWorld with reward function $r_1$ and $r_2$ respectively. Subfigures $c$ and $d$ represent 2 $\times$ 3 GridWorld with reward function $r_1$ and $r_2$ respectively.}
\label{sample_num}
% \vspace{-0.3cm}
\end{figure}

\begin{figure}[htb!]
 \center
\includegraphics[width=\columnwidth]{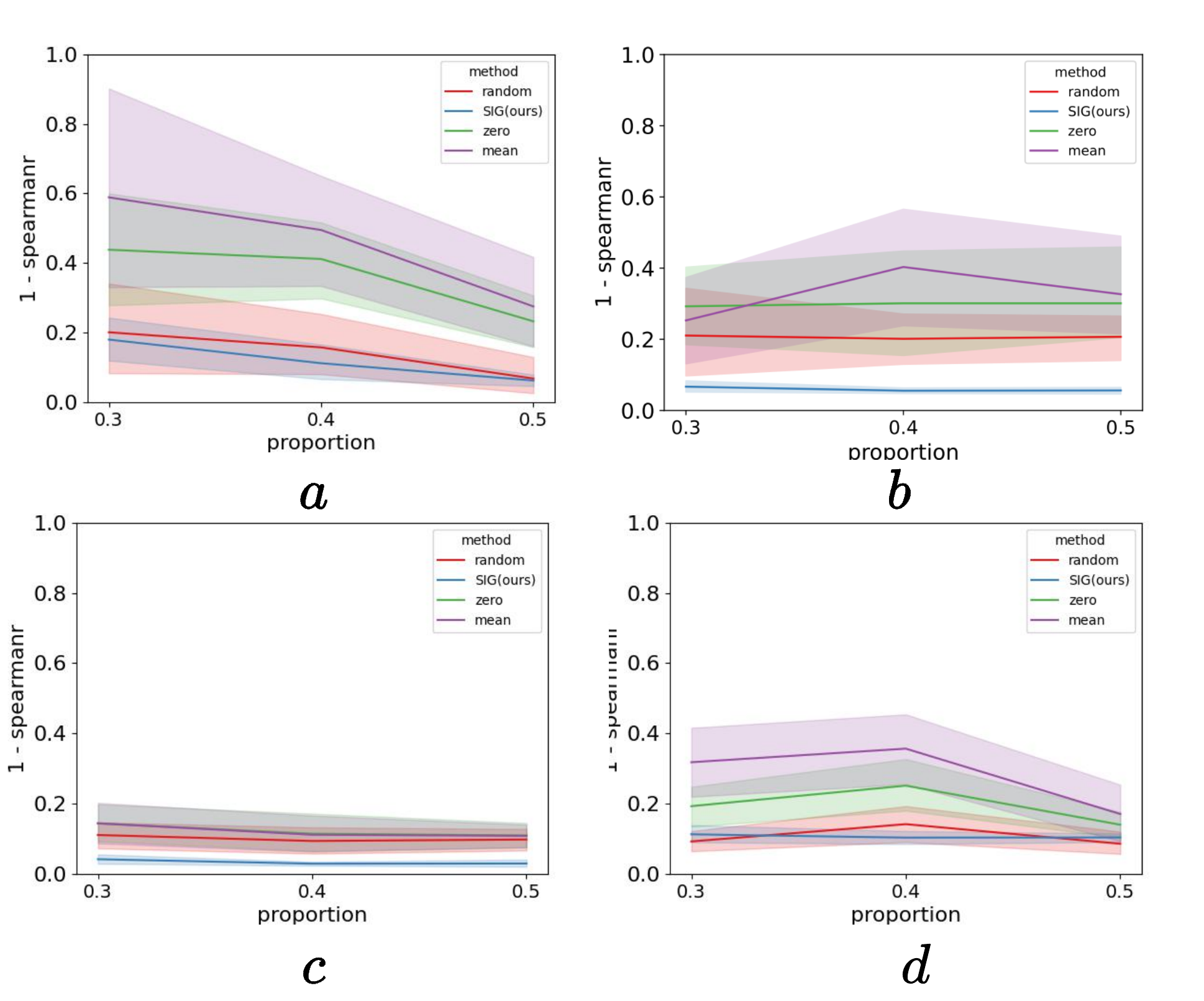}
    \caption{Average 1-Spearmanr metric for hypterparameters $Q$ across four GridWorld situations. Subfigures $a$ and $b$ represents 2 $\times$ 2 GridWorld with reward function $r_1$ and $r_2$ respectivel under fixed $N$ 200. Subfigures $c$ and $d$ represents 2 $\times$ 3 GridWorld with reward function $r_1$ and $r_2$ respectively under fixed $N$ 1500.}
\label{proportion}
% \vspace{-0.3cm}
\end{figure}

\begin{figure}[htb!]
 \center
\includegraphics[width=\columnwidth]{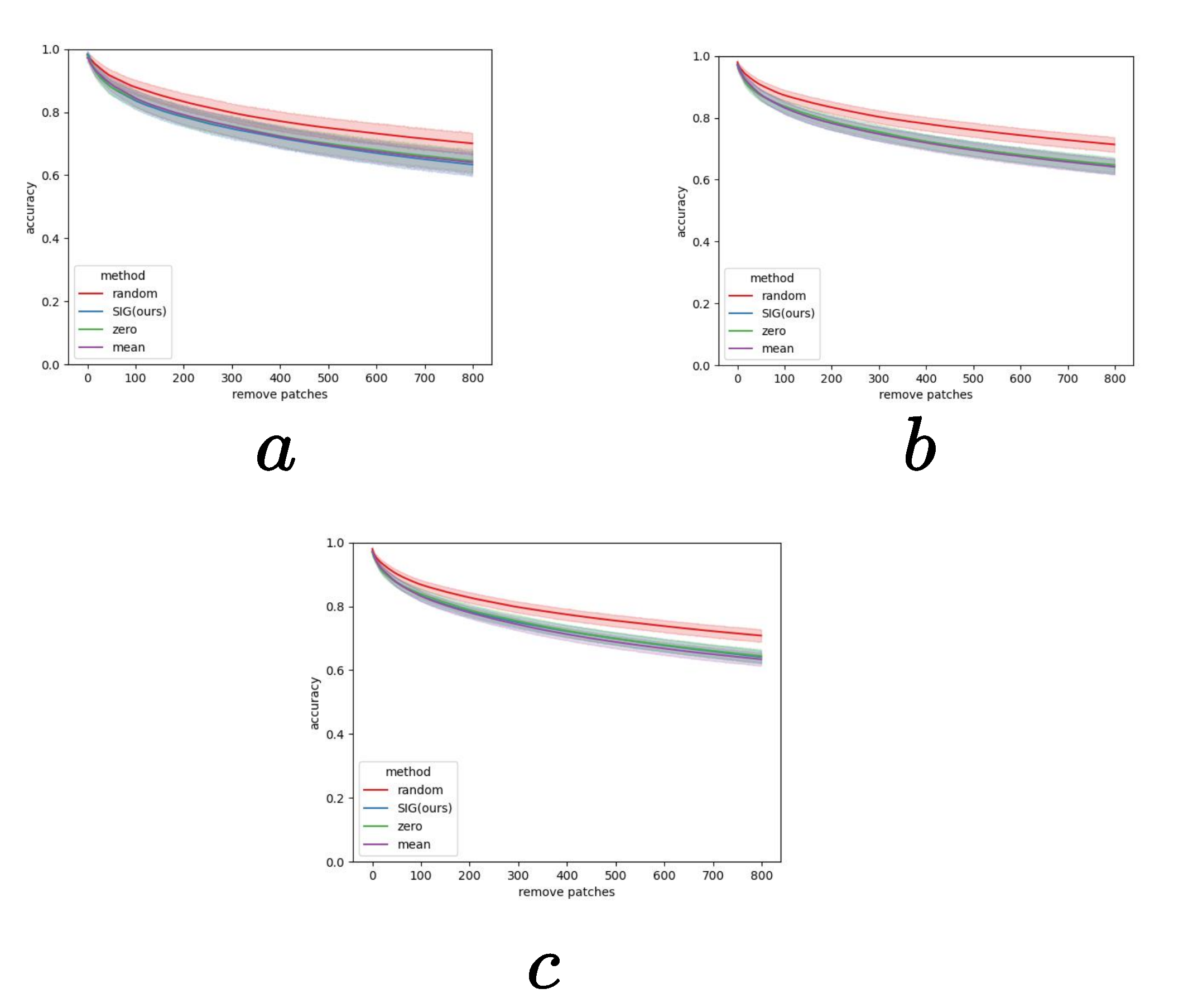}
    \caption{Accuracy metric over ImageNet for Resnet Model. $a$, $b$ and $c$ represent accuracy metrics of 500, 1000, 1500 images respectively.}
\label{images}
% \vspace{-0.3cm
\end{figure}

\begin{figure}[htb!]
 \center
\includegraphics[width=\columnwidth]{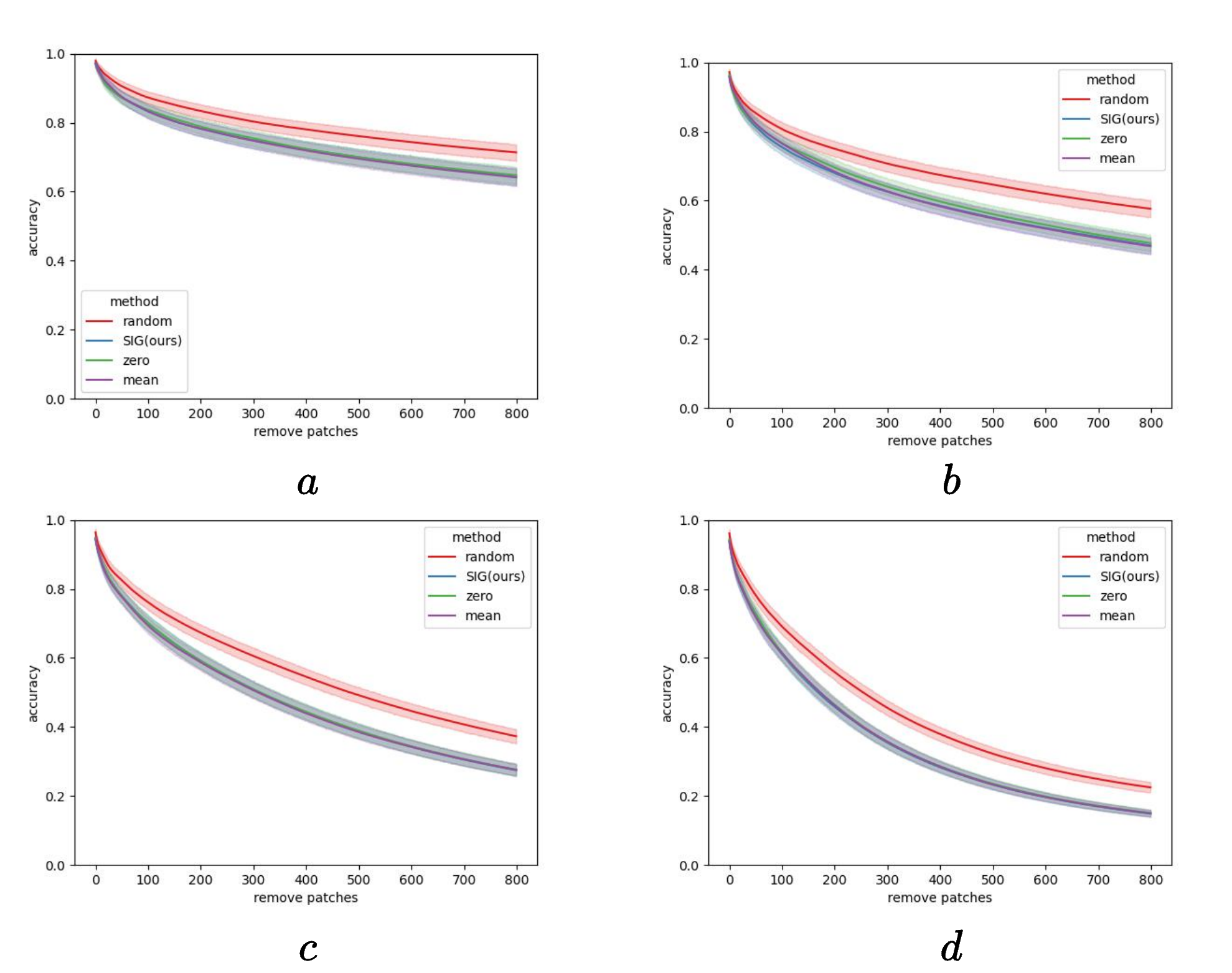}
    \caption{Accuracy metric over ImageNet for Resnet Model across various patches. Subfigures $a$, $b$, $c$, and $d$ correspond to 0, 1, 2, 3 patches under 1000 images respectively.}
\label{patches}
% \vspace{-0.3cm
\end{figure}

 The impact of hypermeters $Q$ is shown in Fig.{~\ref{proportion}}. Further details can be found in code. The experiment reveals that our SIG consistently exhibits notable robustness with respect to $Q$.

\paragraph{Image Classification Task} We further explore the accuracy metric of our SIG with lots of images in the Image Classification Task. Results are presented in Fig.{~\ref{images}}. It's evident that our SIG has almost same performance to zero and mean baseline methods

Moreover, people tend to classify one image based on areas, like heads, bodies, legs, rather than individual pixels. Given this, we are curious to see if pixels around important pixels identified by our SIG align with human intuition. To assess this, we opt to remove patches rather than individual pixels. Patch $n$ is defined as areas of size $(n + 1) \times (n + 1)$ centered on important pixels. As shown in Fig.{~\ref{patches}}, we find that important pixels pinpointed by baseline methods are in alignment with our human intuition. Additionally, our SIG exhibits performance almost to both zero and mean baseline methods. For more details, please refer to the code we have provided.

% \paragraph{Models}
% Architecture, training hyperparameters, sampling, multiple replications, evaluation accuracy ...

% \paragraph{Details and Pitfalls}
% framework, tool, XAI experiemnts hyper-parameters

% \subsection{Supplementary Results}

% \paragraph{Running Time}

% \paragraph{GridWold}

% \paragraph{Expression Code}

% \paragraph{Image Classification}

% \paragraph{Atari}

\end{document}